\documentclass[3p,12pt]{elsarticle}

\usepackage{amsmath,amssymb,amsthm,mathtools}
\usepackage{stmaryrd}
\usepackage{latexsym}
\usepackage{bbm}
\usepackage{subfig}
\usepackage{multirow}
\usepackage{rotating}
\usepackage[section]{placeins} 

\usepackage{url}
\usepackage{xcolor}
\definecolor{newcolor}{rgb}{.8,.349,.1}

\usepackage{algorithmic}
\usepackage{algorithm}

\newcommand{\new}{}

\journal{Fuzzy Sets and Systems}

\newcommand{\reels}{\mathbb{R}}

\def\barcap{\overline{\cap}}

\newcommand{\calA}{{\cal A}}

\newcommand{\calF}{{\cal F}}

\newcommand{\calX}{{\cal X}}

\def\thetah{\widehat{\theta}}

\newcommand{\tF}{{\widetilde{F}}}
\newcommand{\tL}{{\widetilde{L}}}
\newcommand{\tG}{{\widetilde{G}}}
\newcommand{\tH}{{\widetilde{H}}}
\newcommand{\tK}{{\widetilde{K}}}
\newcommand{\tA}{{\widetilde{A}}}
\newcommand{\tB}{{\widetilde{B}}}
\newcommand{\tY}{{\widetilde{Y}}}

\newcommand{\tGamma}{{\widetilde{\Gamma}}}

\newcommand{\tTheta}{{\widetilde{\Theta}}}
\def\tmu{\widetilde{\mu}}

\newcommand{\tm}{{\widetilde{m}}}

\newcommand{\Int}{\textrm{Int}}
\newcommand{\Incl}{\textrm{Incl}}
\newcommand{\is}{\textrm{ is }}

\def\cvgd{\stackrel{d}{\longrightarrow}}

\def\btheta{\boldsymbol{\theta}}

\def\F{{\mathbb F}}

\def\bu{{\boldsymbol{u}}}

\def\bU{{\boldsymbol{U}}}

\def\cut#1#2{{}^#1#2}

\newcommand{\bi}{\begin{itemize}}
\newcommand{\ei}{\end{itemize}}
\newcommand{\be}{\begin{enumerate}}
\newcommand{\ee}{\end{enumerate}}
\newcommand{\bd}{\begin{description}}
\newcommand{\ed}{\end{description}}

\newtheorem{Prop}{Proposition}  
\newtheorem{Rem}{Remark}
\newtheorem{Ex}{Example}

\newtheorem{Req}{Requirement}

\begin{document}
\begin{frontmatter}

\title{Belief functions induced by random fuzzy sets:
\new{A general framework for representing uncertain and fuzzy evidence}}

\author[utc,shu,iuf]{Thierry Den{\oe}ux}
\ead{Thierry.Denoeux@utc.fr}
\address[utc]{Universit\'e de technologie de Compi\`egne, CNRS\\
UMR 7253 Heudiasyc, Compi\`egne, France}
\address[shu]{Shanghai University, UTSEUS, Shanghai, China}
\address[iuf]{Institut universitaire de France, Paris, France}

\begin{abstract}
We revisit Zadeh's notion of ``evidence of the second kind'' and show that it provides the foundation for a general theory of epistemic random fuzzy sets, which generalizes both the Dempster-Shafer theory of belief functions and possibility theory. In this perspective, Dempster-Shafer theory deals with belief functions generated by random sets, while possibility theory deals with belief functions induced by fuzzy sets. The more general theory allows us to represent and combine evidence that is both uncertain and fuzzy. We demonstrate the application of this formalism to statistical inference, and show that it makes it possible to reconcile the possibilistic interpretation of likelihood with Bayesian inference. 
\end{abstract}

\begin{keyword}
Dempster-Shafer theory, evidence theory, possibility theory, fuzzy mass functions, uncertain reasoning, likelihood, estimation, prediction.
\end{keyword}

\end{frontmatter}


\section{Introduction}
\label{sec:intro}

The theory of belief functions developed by Dempster \cite{dempster67a} and Shafer \cite{shafer76} is a theory of uncertain reasoning that puts the emphasis on  the concept of \emph{evidence}. It is based on the representation of elementary pieces of evidence by belief functions (defined as completely monotone set functions) and on their combination by an operator called the product-intersection rule, or Dempster's rule of combination. A belief function can be constructed by comparing a piece evidence to a scale of canonical examples such as randomly coded messages, whose meanings are determined by chance \cite{shafer81}. A belief function on a set $\Theta$ can be seen as being induced by a multi-valued mapping from a probability space to $\Omega$; it is mathematically equivalent to a random set \cite{dempster67a,nguyen78}. As rational beliefs are essentially determined by evidence, the Dempster-Shafer (DS) theory can be regarded as a general framework for reasoning with uncertainty \cite{denoeux20b}. 

Shortly after the introduction of DS theory, Zadeh independently proposed another formalism, called \emph{Possibility Theory} \cite{zadeh78}, in which the concept of ``fuzzy restriction''  plays a central role (see, e.g., \cite{denoeux20a} for a recent review). A fuzzy restriction is typically imposed on a variable $X$ taking values in a set $\Theta$ by a statement of the form ``$X$ is $\tF$'', where $\tF$ is a fuzzy subset of $\Theta$. For instance, the statement ``John is young'' acts as a flexible constraint on  the age of John. Zadeh proposed to identify the membership function of the fuzzy set $\tF$ with a \emph{possibility distribution}: for any $\theta\in \Theta$, $\tF(\theta)$ is then both the degree of membership of $\theta$ to the fuzzy set $\tF$, and the degree of possibility that $X$ takes value $\theta$, knowing that ``$X$ is $\tF$''. The possibility of a subset $A\subseteq \Theta$ is then the supremum of the degrees of possibility $\pi(\theta)$ for all $\theta\in \Theta$. 

The mathematical connections between the two theories were soon pointed out \cite{dubois82a}, namely: a possibility measure corresponds to a belief function induced by a multi-valued mapping that assigns probabilities only to a collection of nested subsets. Such a belief function is said to be \emph{consonant}. From this point of view, the formalism of belief functions is more general than that of possibility theory. However, Dempster's rule does not fit in possibility theory as it does not preserve consonance: the combination of two possibility measures by Dempster's rule is not a possibility measure. Possibility theory has its own conjunctive and disjunctive combination operators based on triangular norms and conorms \cite{dubois88a}:  consequently, it is not a special case of DS theory, but  a stand-alone framework. An open question is whether DS and possibility theories  should be regarded are two ``competing'' models of uncertainty, which can be used interchangeably to  reason with partial information, or whether they should rather be considered as complementary, each theory being suited to represent different kinds of uncertain evidence.  If one adopts the latter view, as I do in this paper, it makes sense to search for a more general theory that encompasses belief functions, possibility measures, and their combination operators as special cases. Interestingly, elements of such a theory already exist, but they have received relatively little attention until now.

The idea of combining DS and possibility theories can first be found in an early paper by Zadeh \cite{zadeh79}. Zadeh calls ``evidence of the second kind'' a pair $(X,\Pi_{(Y\mid X)})$ in which $X$ is a discrete random variable on a set $\Omega$ and  $\Pi_{(Y\mid X)}$ a collection of ``conditioned $\pi$-granules'', i.e., conditional possibility distributions of $Y$ given $X=x$, for all $x\in \Omega$. The probability distribution on $X$ and the conditional possibility distributions together define a \emph{mixture of possibility measures}. If the random variable $X$ is certain, then the mixture has only one component and  it boils down to a single possibility measure representing fuzzy evidence. If the conditional possibility distributions take values in $\{0,1\}$, we essentially get a  DS belief function, representing uncertain evidence. In general, Zadeh's concept of evidence of the second kind  makes it possible to represent evidence that is both fuzzy and uncertain, and it lays down the foundation of a general theory of uncertainty  encompassing possibility theory and DS theory as special cases. This theory, only sketched by Zadeh in \cite{zadeh79} and mostly overlooked in later work,  is further  explored in this paper.

The mathematical structure $(X,\Pi_{(Y\mid X)})$  is actually related to the notion of \emph{random fuzzy set} or \emph{fuzzy random variable}, a concept that also  appeared at the end of the 1970's \cite{feron76, kwakernaak78, kwakernaak79}. \new{Random fuzzy sets have been used with different interpretations that differ from Zadeh's evidence of the second kind, such as a model of a random mechanism for generating fuzzy data \cite{puri86,gil06}, or a representation of imprecise information about the true probability distribution associated with a random experiment \cite{kruse87,couso08}. Baudrit et al \cite{baudrit07} use random fuzzy sets to propagate probabilistic and possibilistic uncertainty in risk analysis from an imprecise probability perspective.  Couso and S\'anchez  \cite{couso11}  propose a framework based on a known probability measure on a  set $\Omega_1$ and a family of conditional possibility measures $\{\Pi(\cdot|\omega): \omega\in \Omega_1\}$ on a set $\Omega_2$, similar to Zadeh's ``conditioned $\pi$-granules''. Independently from the literature on fuzzy random variables,}   a few contributions on ``fuzzy belief structures'', \new{defined as DS mass functions with fuzzy focal sets,} were also published in the 1980's and early 1990's \cite{ishizuka82,yager82,yen90}. Applications to classification, regression and missing data imputation were presented, respectively, in \cite{denoeux01a},  \cite{petit99b,petit04} and \cite{petit98a}. However, in these papers, the formalism of fuzzy belief structures was seen merely as  a fuzzy generalization of DS mass functions, and not as a general concept encompassing DS and possibility theories as special cases, a perspective that is adopted in this paper.

A very important problem for which a general theory of epistemic random fuzzy sets can be useful is \emph{statistical inference}. In \cite{shafer76}, Shafer proposed to treat the relative likelihood function as the contour function or a consonant belief function. This view is quite appealing and it was further justified axiomatically in \cite{denoeux13b}. However, it was later rejected by Shafer \cite{shafer82} because it is not consistent with Dempster's rule of combination: the belief function induced by a joint random sample composed of two independent samples is not the orthogonal sum of the belief functions induced by each of the two samples considered separately. Rather,  relative likelihood functions induced by independent samples must be multiplied and renormalized, an operation that is consistent with a possibilistic interpretation of the likelihood  \cite{smets82,aickin00}. Yet, the limited expressiveness of possibility theory alone does not allow it to represent probability distributions and their combination with likelihoods to yield posterior distributions. The generalized setting proposed in this paper allows us to resolve these difficulties, by viewing the likelihood function as defining a mass function with a single fuzzy focal set. The resulting model still generalizes Bayesian inference as did Shafer's original model, while being consistent with Dempster's rule of combination. To highlight the basic principles without being distracted by mathematical intricacies,  we will assume the variables to be defined on finite domains throughout this paper. The case of  infinite domains will be addressed in a companion paper.

The rest of the paper is organized as follows. DS  and possibility  theories are first recalled in Section \ref{sec:background}, where some new definitions and results are also given. The more general  random fuzzy set model is then  exposed in Section \ref{sec:fuzzyDS}. Finally, the application to statistical inference is addressed in Section \ref{sec:stat}, and Section \ref{sec:concl}  concludes the paper.

\section{DS and possibility theories}
\label{sec:background}

In this section, we consider a variable $\btheta$ taking values in a finite set $\Theta$, and we review different classical models for representing and combining  evidence about $\btheta$. The case of \new{logical} evidence is first recalled in Section \ref{subsec:certain}. This simple case can be generalized  to \emph{uncertain} evidence, leading to DS theory, or to \emph{fuzzy} evidence, resulting in possibility theory. These two theories are reviewed, respectively, in  Sections \ref{subsec:DS} and  \ref{subsec:poss}. 

\subsection{\new{Logical} evidence}
\label{subsec:certain}

 Let $F$ be a subset of $\Theta$, and assume that we receive a piece of evidence  that us that $\btheta \in F$ for sure, and nothing more. \new{We call such evidence ``logical'' because it consists of a proposition that is known to be true.} Given this evidence, a proposition ``$\btheta \in A$'' for some $A\subseteq \Theta$ is \emph{possible} if and only if  $F\cap A\neq \emptyset$, and it is \emph{certain} if and only if $F\subseteq A$. We can define two set functions $\Pi_F$ and $N_F$ from the power set $2^\Theta$ to $\{0,1\}$, as
\begin{equation*}
\Pi_F(A):=I(F \cap A\neq \emptyset)
\end{equation*}
and
\[
N_F(F):=I(F \subseteq A),
\]
where $I(\cdot)$ is the indicator function, \new{which returns 1 if its argument is true, and 0 otherwise} \cite{dubois98b}. We can remark that $\Pi_F(A)=1$ if and only if there is some $\theta \in A$ that belongs to $F$. We can thus write, equivalently,
\begin{equation}
\Pi_F(A)= \max_{\theta\in A} F(\theta),
\end{equation}
where $F(\cdot)$ denotes the characteristic function of set $F$. The possibility value of $\{\theta\}$, denoted by $\pi_F(\theta)$, is 
\begin{equation}
\label{eq:Boolean_pi}
\pi_F(\theta):= \Pi_F(\{\theta\})=F(\theta).
\end{equation}
Furthermore, the proposition $F\subseteq A$ is equivalent to $F\cap A^c=\emptyset$, where $A^c$ denotes the complement of $A$. Consequently, we have
\begin{equation}
\label{eq:NPi}
N_F(A)=1-\Pi_F(A^c).
\end{equation}

It is clear that, for any two subsets $A$ and $B$ of $\Theta$, the following equalities hold:
\begin{equation}
\label{eq:Boolmax}
\Pi_F(A\cup B)=\Pi_F(A) \vee \Pi_F(B)
\end{equation}
and
\begin{equation}
\label{eq:Boolmin}
N_F(A\cap B)= N_F(A) \wedge N_F(B),
\end{equation}
where $\vee$ and $\wedge$ denote, respectively, the maximum  and  minimum operators. Functions $\Pi_F$ and $N_F$ can be called, respectively, \emph{Boolean possibility} and \emph{necessity} measures, and function $\pi_F: \Theta \rightarrow \{0,1\}$ can be called a \emph{Boolean possibility distribution}.

A subset $A\subseteq \Theta$ is possible if at least one element $\theta\in A$ is possible (i.e., belongs to $F$). This is captured by function $\Pi_F$. We can also define the stronger notion of \emph{guaranteed possibility} \cite{dubois98b}: $A\subseteq \Theta$ is ``guaranteed possible'' if \emph{all} its elements are possible, i.e., if $A \subseteq F$. This notion is captured by the guaranteed possibility function
\begin{equation}
\label{eq:Gpos}
\Delta_F(A):=I(A \subseteq F) =  \min_{\theta\in A} F(\theta).
\end{equation}
We can also define a dual ``potential certainty'' function, $\nabla_F(A)=1-\Delta_F(A^c)$, which equals one if and only if there is some $\theta$ outside $A$ that is not possible, i.e., if $A\cup F \neq \Theta$.

If we now have two pieces of evidence telling us that $\btheta\in F$ for sure and $\btheta\in G$ for sure, and if we consider that both sources can be trusted, then we can infer that $\btheta\in F\cap G$. The combined Boolean possibility distribution is, thus,
\begin{equation}
\pi_{F\cap G}(\theta):= \pi_F(\theta) \wedge  \pi_G(\theta).
\end{equation}
If, on the other hand, we only consider that at least one of the two sources can be trusted, then we can infer that  $\btheta\in F\cup G$, and  the combined Boolean possibility distribution is
\begin{equation}
\pi_{F\cup G}(\theta):= \pi_F(\theta) \vee  \pi_G(\theta).
\end{equation}

This simple model can be extended in two ways, by allowing the piece of information to be uncertain, or fuzzy. These two extensions correspond, respectively, to Dempster-Shafer (DS) and Possibility theories, which are recalled in   Sections \ref{subsec:DS} and \ref{subsec:poss} below. 

\subsection{Uncertain evidence: DS theory}
\label{subsec:DS}

Let us now assume that we receive a piece of evidence that can be interpreted in different ways, with given probabilities \cite{shafer76,denoeux20b}. Let $\Omega$ be the   set of interpretations, assumed to be finite. If interpretation $\omega\in \Omega$ holds, then the evidence tells us that $\btheta$ belongs to some nonempty subset $\Gamma(\omega) \subseteq \Theta$, and nothing more. We further assume that we can assess probabilities on $\Omega$, on which we define a probability measure $P$. The tuple $(\Omega,2^\Omega,P,\Gamma)$, where $\Gamma$ is a mapping from $\Omega$ to $2^\Theta$ is a \emph{random set}. We  define the corresponding \emph{mass function} $m: 2^\Theta$ to $[0,1]$ as
\[
m(A):=P\left(\{\omega \in \Omega : \Gamma(\omega)=A\}\right),
\]
for all nonempty subset $A\subseteq \Theta$, and $m(\emptyset)=0$. A subset $F$ such that $m(F)>0$ is called a \emph{focal set} of $m$. We denote  the focal sets by $F_1,\ldots,F_f$, and their masses by $m_i:=m(F_i)$ for $i=1,\ldots,f$. A mass function is said to be \emph{logical} if it has only one focal set, and \emph{Bayesian} if all of its focal sets are singletons.  

If interpretation $\omega$ holds, we know that $\btheta\in \Gamma(\omega)$. The possibility and necessity that $\btheta \in A$ are then, respectively, $\Pi_{\Gamma(\omega)}(A)$ and $N_{\Gamma(\omega)}(A)$. We can then compute the \emph{expected possibility} and the \emph{expected necessity} of the proposition ``$\btheta \in A$'' as, respectively,
\begin{equation}
\label{eq:defPl}
Pl_m(A) = \sum_{\omega \in \Omega} P(\{\omega\}) \Pi_{\Gamma(\omega)}(A) = \sum_{i=1}^f m_i \Pi_{F_i}(A) = \sum_{i=1}^f m_i I(F_i \cap A\neq \emptyset)
\end{equation}
and
\begin{equation}
\label{eq:defBel}
Bel_m(A) = \sum_{\omega \in \Omega} P(\{\omega\}) N_{\Gamma(\omega)}(A) =\sum_{i=1}^f m_i N_{F_i}(A)=\sum_{i=1}^f m_i I(F_i \subseteq  A).
\end{equation}
These two functions are called, respectively,  \emph{plausibility} and a \emph{belief}  functions. If $m$ has only one focal set $F$, it is clear that functions $Pl_m$ and $Bel_m$ boil down, respectively, to the Boolean possibility and necessity functions $\Pi_F$ and $N_F$ reviewed in Section \ref{subsec:certain}. In the general case, they are mixtures of such functions. The restriction of function $Pl_m$ to singletons, denoted as $pl_m(\theta):=Pl_m(\{\theta\})$, is called the\emph{ contour function} associated to $m$. It is equal to function $\pi_F$ when $m$ has only one focal set $F$.

Functions $Pl_m$ and $Bel_m$ are linked by the duality relation $Bel_m(A)=1-Pl_m(A^c)$, which is a direct consequence of  \eqref{eq:NPi}. Shafer \cite{shafer76} shows that a mapping $Bel: 2^\Theta \rightarrow [0,1]$ can be written in the form \eqref{eq:defBel} if and only if $Bel(\emptyset)=0$, $Bel(\Theta)=1$, and $Bel$ is completely monotone, i.e., 
\begin{equation}
\label{eq:monotone}
Bel\left( \bigcup_{i=1}^k A_i\right) \ge \sum_{\emptyset\neq I \subseteq \{1,\ldots,k\}} (-1)^{|I|+1} Bel\left( \bigcap_{i\in I} A_i \right)
\end{equation}
for any $k\ge 2$ and any collection $A_1,\ldots,A_k$ of subsets of $\Theta$. 

In addition to the expected possibility and necessity, we can also compute the\emph{ expected guaranteed possibility} of the proposition ``$\btheta \in A$''  as
\begin{equation}
\label{eq:defQ}
Q_m(A) = \sum_{\omega \in \Omega} P(\{\omega\}) \Delta_{\Gamma(\omega)}(A) =\sum_{i=1}^f m_i \Delta_{F_i}(A)=\sum_{i=1}^f m_i I(A \subseteq F_i).
\end{equation}
Function $Q_m$ is called the \emph{commonality function} induced by $m$ \cite{shafer76}; it plays an important role in DS theory, as will be shown below. The dual function $\nabla_m(A)= 1-Q_m(A^c)$ could also be defined, but its interpretation in DS theory is less clear. Functions $m$, $Bel_m$, $Pl_m$ and $Q_m$ are in one-to-one correspondence and any one of them allows us to recover the other three. They can thus be considered as different facets of the same information.

\paragraph{Dempster's rule} Let us now assume that we have two mass functions $m_1$ and $m_2$ on $\Theta$ induced by \emph{independent} pieces of evidence, i.e., by independent random sets $(\Omega_1,2^{\Omega_1},P_1,\Gamma_1)$ and $(\Omega_2,2^{\Omega_2},P_2,\Gamma_2)$. We further assume that both pieces of evidence are reliable, i.e., if the pair of interpretations $(\omega_1,\omega_2)\in \Omega_1\times \Omega_2$ holds, then we know for sure that $\btheta\in \Gamma_\cap(\omega_1,\omega_2):=\Gamma_1(\omega_1)\cap\Gamma_2(\omega_2)$, provided that $\Gamma_1(\omega_1)\cap\Gamma_2(\omega_2)\neq \emptyset$. To compute the probability that \new{a particular event in $\Omega_1\times\Omega_2$} holds, we need to compute the product measure $\new{P:=}P_1\otimes P_2$ and to condition it on the event 
\begin{equation}
\label{eq:Theta_star}
\Theta^*:=\{(\omega_1,\omega_2) \in \Omega_1\times \Omega_2 : \Gamma_\cap(\omega_1,\omega_2)\neq\emptyset\}.
\end{equation}
The \emph{orthogonal sum} \cite{shafer76} of $m_1$ and $m_2$ is then defined, for all $A\in 2^\Theta$, as
\begin{subequations}
\label{eq:dempster}
\begin{align}
(m_1\oplus m_2)(A)&:= P(\{(\omega_1,\omega_2)\in\Omega_1\times \Omega_2: \Gamma_\cap(\omega_1,\omega_2)=A\} \mid  \Theta^*)\\
&= \frac{P(\{(\omega_1,\omega_2)\in\Omega_1\times \Omega_2: \Gamma_\cap(\omega_1,\omega_2)=A\} \cap \Theta^*)}{P(\Theta^*)}\\
&=\begin{dcases}
0 &  \text{if } A=\emptyset\\
\frac{\sum_{F\cap G=A} m_1(F) m_2(G)}{\sum_{F\cap G\neq \emptyset} m_1(F) m_2(G)} & \text{otherwise, }
\end{dcases}
\end{align}
\end{subequations}
which is well defined on the condition that $P(\Theta^*)>0$. The binary operation $\oplus$ is called \emph{Dempster's rule}. It is commutative and associative. The quantity
\[
\kappa:=1- P(\Theta^*)=\sum_{F\cap G= \emptyset} m_1(F) m_2(G)
\]
is called the \emph{degree of conflict} between $m_1$ and $m_2$. 

The following two propositions are important in practice \cite{shafer76}.

\begin{Prop}
\label{prop:combQ}
Let $m_1$ and $m_2$ be two mass functions on $\Theta$. The commonality function $Q_{m_1\oplus m_2}$ corresponding to the orthogonal sum $m_1\oplus m_2$ is proportional to the product of the commonality functions $Q_{m_1}$ and $Q_{m_2}$ corresponding to $m_1$ and $m_2$:
\[
Q_{m_1\oplus m_2}= \frac{1}{1-\kappa} Q_{m_1} Q_{m_2}.
\]
As $Q_m(\{\theta\})=pl_m(\theta)$ for all $\theta\in\Theta$, we also have $pl_{m_1\oplus m_2}= (1-\kappa)^{-1} pl_{m_1} pl_{m_2}$.
\end{Prop}

\begin{Prop}
\label{prop:combBayes}
Let $m_1$ be a mass function and let $m_2$ be a Bayesian mass function. Then, the mass function $m_1\oplus m_2$ is Bayesian; it is defined as
\[
(m_1\oplus m_2)(\{\theta\})=\frac{pl_{m_1}(\theta) m_2(\{\theta\})}{\sum_{\theta'\in\Theta}pl_{m_1}(\theta') m_2(\{\theta'\})}
\] 
for all $\theta\in\Theta$.
\end{Prop}

As mentioned above, Dempster's rule is based on the assumption that both pieces of evidence are reliable. We can also weaken this assumption  and only assume that \emph{at least one of them} is reliable \cite{dubois86a,smets93b}. If the pair of interpretations $(\omega_1,\omega_2)\in \Omega_1\times \Omega_2$ holds,  we can then deduce that $\btheta\in \Gamma_\cup(\omega_1,\omega_2):=\Gamma_1(\omega_1)\cup\Gamma_2(\omega_2)$. Still assuming the two pieces of evidence to be independent, we get the combined mass function $m_1\Cup m_2$ defined,  for all $A\in 2^\Theta$, as
\begin{subequations}
\label{eq:disj_rule}
\begin{align}
(m_1\Cup m_2)(A)&:= P(\{(\omega_1,\omega_2)\in\Omega_1\times \Omega_2: \Gamma_\cup(\omega_1,\omega_2)=A\})\\
&=\sum_{F\cup G=A} m_1(F) m_2(G).
\end{align}
\end{subequations}
We note that  normalization is not needed in this case, as $\Gamma_\cup(\omega_1,\omega_2)$ cannot be empty as long as $\Gamma_1(\omega_1)$ and $\Gamma_2(\omega_2)$ are nonempty.

\paragraph{Degree of belief in a fuzzy event}

The notions of a fuzzy event and its probability were first defined by Zadeh \cite{zadeh68}. 
Given a probability space $(\Theta,\calA,P)$, a \emph{fuzzy event} is a fuzzy subset $\tA$ of $\Theta$ with measurable membership function, and the probability of $\tA$ is defined as the expectation of its membership function. When $\Theta$ is finite, as assumed throughout his paper, \new{and $\calA=2^\Theta$, the measurability of $\tA$ is always satisfied} and we have
\begin{equation}
\label{eq:probaFuzzy}
P(\tA):=\sum_{\theta\in\Theta} P(\{\theta\}) \tA(\theta),
\end{equation}
where $\tA(\theta)$ denotes the degree of membership of $\theta$ in the fuzzy set $\tA$. It can be shown \cite{dubois00c} that $P(\tA)$ can also be written as:
\begin{equation}
\label{eq:probaFuzzy1}
P(\tA)=\int_0^1 P(\cut\alpha\tA)d\alpha,
\end{equation}
where ${}^\alpha\tA=\{\theta\in \Theta : \tA(\theta)\ge \alpha\}$ is the $\alpha$-cut of $\tA$.

Smets \cite{smets81} extended this definition to the case where uncertainty on $\Theta$ is defined by a mass function $m$. He defined the \emph{degrees of belief and plausibility} of  fuzzy event $\tA$ as, respectively, the lower and upper expectations of its membership function:
\begin{subequations}
\label{eq:fuzzyBel}
\begin{align}
Bel_m(\tA)&=\sum_{i=1}^f m_i \min_{\theta\in F_i} \tA(\theta)\\
Pl_m(\tA)&=\sum_{i=1}^f m_i \max_{\theta\in F_i} \tA(\theta).
\end{align}
\end{subequations}
Similarly to \eqref{eq:probaFuzzy1}, we have
\begin{subequations}
\label{eq:BelPlChoquet}
\begin{align}
Bel_m(\tA)&=\int_0^1 Bel_m(\cut\alpha\tA)d\alpha\\
Pl_m(\tA)&=\int_0^1 Pl_m(\cut\alpha\tA)d\alpha.
\end{align}
\end{subequations}
We note that $Bel_m(\tA)$ and $Pl_m(\tA)$ are the Choquet integrals of $\tA(\cdot)$ with respect to \new{$Bel_m$ and $Pl_m$}, respectively. It is clear that definition \eqref{eq:fuzzyBel} coincides with \eqref{eq:probaFuzzy} when $m$ is Bayesian.  Let $\calF(\Theta)$ be the set of all fuzzy subsets of $\Theta$. It becomes a lattice when equipped with fuzzy intersection and union defined, respectively, as 
\begin{subequations}
\begin{equation}
\label{eq:fuzzyinter}
(\tA\wedge\tB)(\theta)=\tA(\theta) \wedge \tB(\theta)
\end{equation}
and
\begin{equation}
(\tA\vee\tB)(\theta)=\tA(\theta) \vee \tB(\theta)
\end{equation}
\end{subequations}
for all $\theta\in \Theta$, where $\vee$ and $\wedge$ denote, respectively, the minimum and the maximum.  As shown by Smets \cite{smets81}, the mapping $Bel$ from $\calF(\Theta)$ to $[0,1]$ defined by \eqref{eq:fuzzyBel} is a belief function on $(\calF(\Theta),\wedge,\vee)$, i.e., it verifies the following inequalities for  any $k\ge 2$ and countable collection  $\tA_1,\ldots,\tA_k$ of fuzzy subsets of $\Theta$:
\begin{equation}
\label{eq:monotone1}
Bel_m\left( \bigvee_{i=1}^k \tA_i\right) \ge \sum_{\emptyset\neq I \subseteq \{1,\ldots,k\}} (-1)^{|I|+1} Bel_m\left( \bigwedge_{i\in I} \tA_i \right).
\end{equation}

Although Smets did not consider extending the commonality function to fuzzy events, this can also be done as follows:
\begin{subequations}
\label{eq:mQfuzzy}
\begin{align}
Q_m(\tA)&:=\int_0^1 Q(\cut\alpha\tA)d\alpha\\
&=\int_0^1 \left(\sum_{i=1}^f m_i I(\cut\alpha\tA \subseteq F_i)\right)d\alpha\\
&= \sum_{i=1}^f m_i \int_0^1 I(\cut\alpha\tA \subseteq F_i)d\alpha\\
&= \sum_{i=1}^f m_i \left(1-\max_{\theta\not\in F_i} \tA(\theta)\right).
\end{align}
\end{subequations}

\subsection{Fuzzy evidence: possibility theory}  
\label{subsec:poss}

As recalled in Section \ref{sec:intro}, possibility theory introduced by Zadeh \cite{zadeh78} extends the simple model outlined in Section \ref{subsec:certain} by considering statements  of the form ``$\btheta \is \tF$'', where $\tF$ si a normal fuzzy subset of $\Theta$, i.e., a fuzzy subset verifying $\tF(\theta)=1$ for some $\theta\in\Theta$ (see also \cite{dubois88a,dubois98b,denoeux20a}).  \new{Such fuzzy statements are good representations of evidence that can be fully trusted but that does not have a precise meaning because, e.g., it is conveyed through natural language. For instance, I may know from a fully reliable witness that ``John is old''.} Zadeh \cite{zadeh78} defines the \emph{possibility measure} $\Pi_\tF: 2^\Theta \rightarrow [0,1]$ induced by the piece of information ``$\btheta \is \tF$''  as
\begin{equation}
\label{eq:poss}
\Pi_\tF(A):=\max_{\theta\in A} \tF(\theta),
\end{equation}
which generalizes \eqref{eq:Boolean_pi}. The dual \emph{necessity measure} is defined as
\begin{equation}
\label{eq:nec}
N_\tF(A) := 1-\Pi_\tF(A^c)= \min_{\theta\not\in A} \left[1-\tF(\theta)\right]
\end{equation}
for all $A\subseteq \Theta$, which extends \eqref{eq:NPi}. The corresponding \emph{possibility distribution} is defined as
\[
\pi_\tF(\theta):=\Pi_\tF(\{\theta\})=\tF(\theta)
\]
for all $\theta\in \Theta$. It is, thus, numerically equal to $\tF$.

It can easy be seen that
\begin{subequations}
\label{eq:poss_nec1}
\begin{equation}
\label{eq:poss1}
\Pi_\tF(A)= \int_0^1 \Pi_{\cut\alpha\tF}(A)d\alpha
\end{equation}
and
\begin{equation}
\label{eq:nec1}
N_\tF(A)= \int_0^1 N_{\cut\alpha\tF}(A)d\alpha.
\end{equation}
\end{subequations}

Trivially, we have $\Pi_\tF(\emptyset)=N_\tF(\emptyset)=0$, $\Pi_\tF(\Theta)=N_\tF(\Theta)=1$, and  $N_\tF(A) \le \Pi_\tF(A)$ for all $A\subseteq \Theta$. Furthermore, the following equalities hold:
\begin{subequations}
\label{eq:poss_nec}
\begin{equation}
\label{eq:Poss_max}
\Pi_\tF(A\cup B)=\Pi_\tF(A) \vee \Pi_\tF(B)
\end{equation}
and 
\begin{equation}
N_\tF(A\cap B)=N_\tF(A) \wedge N_\tF(B)
\end{equation}
\end{subequations}
for all $A,B \subseteq \Theta$. These equalities generalize \eqref{eq:Boolmax} and \eqref{eq:Boolmin}.  

In addition to the possibility and necessity measures defined by \eqref{eq:poss_nec}, Dubois and Prade \cite{dubois98b} also generalize the notion of \emph{guaranteed possibility} \eqref{eq:Gpos} as
\begin{equation}
\label{eq:GP}
\Delta_\tF(A) := \min_{\theta\in A} \tF(\theta) =\int_0^1 \Delta_{\cut\alpha\tF}(A)d\alpha,
\end{equation}
and that of \emph{potential certainty} as
\begin{equation}
\label{eq:PC}
\nabla_\tF(A) := 1- \Delta_\tF(A^c)= \max_{\theta\not\in A} \left[1-\tF(\theta)\right]=\int_0^1 \nabla_{\cut\alpha\tF}(A)d\alpha.
\end{equation}
The quantity $\Delta_\tF(A)$ measures the extent to which \emph{all} values in $A$ are actually possible given the statement ``$\btheta \text{ is } \tF$'', while $\nabla_\tF(A)$ measures the extent to which \emph{at least one} value $\theta$ outside $A$ has a low degree of possibility. Clearly, the following equality holds for any subsets $A$ and $B$ of $\Theta$:
\begin{equation}
\label{eq:Delta_min}
\Delta_\tF(A \cup B) = \Delta_\tF(A) \wedge  \Delta_\tF(B).
\end{equation}

\paragraph{Combination of possibility measures} Let us now assume that  we receive two pieces of evidence  that tell us that ``$\btheta \is \tF$'' and ``$\btheta \is \tG$'', where $\tF$ and $\tG$ are two normal fuzzy subsets of $\Theta$. These fuzzy sets induce two possibility distributions   $\pi_\tF$ and $\pi_\tG$. How should they be combined? If both sources are assumed to be reliable, then it makes sense to infer that ``$\btheta \is \tF \cap \tG$'', where $\tF \cap \tG$ denotes the intersection of fuzzy sets $\tF$ and $\tG$. There are, however, two difficulties. First, there are several possible definitions of fuzzy set intersection. Zadeh \cite{zadeh65} defines the intersection of fuzzy sets $\tF$ and $\tG$ as
\[
(\tF\wedge\tG)(\theta)=\tF(\theta) \wedge \tG(\theta),
\]
but he also proposes the following alternative definition: 
\[
(\tF \cdot \tG)(\theta)=\tF(\theta) \cdot  \tG(\theta).
\]
It is clear that both definitions extend the usual set intersection. Later, these two definitions have been generalized as $(\tF\cap_\top\tG)(\theta):=\tF(\theta) \top \tG(\theta)$, where $\top$ is a triangular norm (or t-norm for short) \cite{dubois00a}.  The minimum is the largest t-norm. It is  consistent with the definition of inclusion as $\tF \subseteq \tG$ iff $\tF \le \tG$: $\tF\wedge\tG$ is then the largest fuzzy set included in $\tF$ and $\tG$. It is also idempotent, i.e., $\tF\wedge\tF=\tF$; consequently, there is no reinforcement effect when two pieces of evidence are identical, which makes minimum-intersection combination useful for combining evidence from dependent and possibly redundant sources. In contrast, product intersection has a reinforcement effect that is appropriate when the sources are assumed to be independent \cite[page 352]{dubois99}.

The second difficulty when combining possibility distributions conjunctively is that the intersection of two fuzzy sets $\tF$ and $\tG$ may not be normal, and some normalization step has to take place. The standard normalization procedure divides $\tF\cap \tG$ by its height (i.e., supremum). As a consequence, the associativity property is usually lost. However,  the normalized product intersection, defined as 
\begin{equation}
\label{eq:norm_prod}
\tF\varodot\tG := \frac{\tF \cdot \tG}{h(\tF \cdot  \tG)}, 
\end{equation}
is associative \cite{dubois99}. We give a simple proof of this well-known result below for completeness.

\begin{Prop}
\label{prop:product_intersection}
Let $\tF$, $\tG$ and $\tH$ be fuzzy subsets of $\Theta$, and let $\varodot$ denote the normalized intersection based on the product t-norm. Then,
\[
(\tF \varodot \tG) \varodot \tH = \tF \varodot (\tG\varodot \tH).
\]
\end{Prop}
\begin{proof}
The key property of the product-based intersection is that, for any $\alpha\ge 0$, $h((\alpha\tF) \cdot  \tG)=h(\tF \cdot (\alpha\tG))=\alpha h(\tF \cdot \tG)$. Using this property, we have
\[
(\tF \varodot \tG) \varodot \tH= \frac{\frac{\tF\cdot\tG}{h(\tF\tG)} \cdot \tH}{h\left(\frac{\tF\cdot\tG}{h(\tF\cdot\tG)} \cdot \tH\right)}=\frac{\tF\cdot\tG\cdot\tH/h(\tF\cdot\tG)}{h(\tF\cdot\tG\cdot\tH)/h(\tF\cdot\tG)}=\frac{\tF\cdot\tG\cdot\tH}{h(\tF\cdot\tG\cdot\tH)}.
\]
and, by the commutativity of $\varodot$, 
\[
\tF \varodot (\tG\varodot \tH)= (\tG\varodot \tH) \varodot \tF= \frac{\tF\cdot\tG\cdot\tH}{h(\tF\cdot\tG\cdot\tH)}.
\]
\end{proof}

\new{The normalized intersection \eqref{eq:norm_prod} is only defined if  $h(\tF \cdot  \tG)>0$. A value of  $h(\tF \cdot  \tG)$ close to zero signals conflicting  evidence. Several authors have warned agains the use of normalized intersection in this case (see, e.g., \cite[page 354]{dubois99}). Indeed, the division by a small number  in \eqref{eq:norm_prod} may result in high sensitivity of the combination to small changes in $\tF$ or $\tG$. Actually,  conflict in the evidence may sometimes lead us to questioning its reliability.} If we  only assume that at least one source is reliable,  we can only deduce that ``$\btheta \is \tF \cup \tG$'', where $\tF \cup \tG$ denotes the union of fuzzy sets $\tF$ and $\tG$. The union $\tF \cup \tG$ is usually defined as  $(\tF \cup_\bot  \tG)(\theta):=\tF(\theta) \bot \tG(\theta)$, where $\bot$ is a t-conorm. To each t-norm $\top$ corresponds a dual t-conorm $\bot$ defined as $u \bot v = 1-[(1-u) \top (1-v)]$. The t-conorms corresponding to the minimum and the product are, respectively, the maximum and the probabilistic sum $u \bot v = u + v -uv$.

\paragraph{Relation between possibility measures and belief functions} 

It can easily be shown  that the mapping $N_\tF: 2^\Omega \rightarrow [0,1]$  is completely monotone \eqref{eq:monotone},  i.e., it is a belief function, and $\Pi_\tF$ is the dual plausibility function \cite{dubois98b}. Actually,  a belief function $Bel_m$ and the associated plausibility function $Pl_m$ verify the equalities
\[
Bel_m(A\cap B)=Bel_m(A) \vee Bel_m(B)
\]
and
\[
Pl_m(A\cup B)=Pl_m(A) \vee Pl(_mB)
\]
iff the corresponding mass function $m$ is  \emph{consonant} \cite{shafer76}, i.e., if for any pair of focal sets $F_i$ and $F_j$, we have $F_i \subset F_j$ or $F_j \subset F_i$.  To each possibility distribution $\pi_\tF$ thus corresponds a unique consonant mass function $m_\tF$, which can be recovered as follows \cite{dubois82a}. Let $\theta_1,\ldots,\theta_q$ denote the elements of $\Theta$, assumed to be indexed in such a way that 
\[
1=\pi_\tF(\theta_1)\ge \pi_\tF(\theta_2) \ge \ldots \ge \pi(_\tF\theta_q).
\]
The corresponding mass function $m_\tF$ is given by
\begin{equation}
\label{eq:consonant}
m_\tF(\{\theta_1,\ldots,\theta_i\})=\pi_\tF(\theta_i)- \pi_\tF(\theta_{i+1})
\end{equation}
for $i=1,\ldots,q-1$, and $m_\tF(\Theta)=\pi_\tF(\theta_q)$. It can easily be checked that functions $Bel_{m_\tF}$,  $Pl_{m_\tF}$, $pl_{m_\tF}$ and $Q_{m_\tF}$ are equal, respectively, to $N_\tF$, $\Pi_\tF$, $\pi_\tF$ and $\Delta_\tF$.

However, the combination of two consonant mass functions by Dempster's rule is no longer consonant. To combine   two consonant belief functions $Bel_1$ and $Bel_2$ induced by independent sources, we must, therefore, consider the evidence on which they are based:
\bi
\item If they are based on fully reliable but vague (fuzzy) evidence such as $\btheta \is \tF$ and $\btheta \is \tG$, then they should be combined  by \new{a conjunctive operator of possibility theory; the normalized product intersection \eqref{eq:norm_prod} seems to be a good choice as it is associative};
\item If they are based on uncertain but crisp (nonfuzzy) evidence pointing to consonant focal sets, then the corresponding consonant mass functions should be combined by Dempster's rule \eqref{eq:dempster}.
\ei
The two combination mechanisms yield different results but, as a consequence of Proposition \ref{prop:combQ}, the contour functions are proportional, as  illustrated by the following example.

\begin{Ex}
\label{ex:consonant}
Let $\Theta=\{\theta_1,\theta_2,\theta_3,\theta_4\}$. Assume that we receive the following pieces of evidence from two independent and reliable sources: 
\bi
\item First piece of evidence: ``$\btheta$ is $\tF$'', with 
\[
\tF=\left\{\frac{\theta_1}{0.5}, \frac{\theta_2}{1}, \frac{\theta_3}{0.8}, \frac{\theta_4}{0.3}\right\};
\]
\item Second piece of evidence: ``$\btheta$ is $\tG$'', with 
\[
\tG=\left\{\frac{\theta_1}{0.3}, \frac{\theta_2}{0.7}, \frac{\theta_3}{1}, \frac{\theta_4}{0.2}\right\};
\]
\ei 
These pieces of evidence can be represented by possibility distributions $\pi_\tF$ and $\pi_\tG$ and combined by  \eqref{eq:norm_prod}, resulting in the combined possibility distribution $\pi_{\tF\varodot\tG}$, with  $\tF\varodot\tG$ equal to
\[
\tF\varodot\tG=\left\{\frac{\theta_1}{0.15/0.8}, \frac{\theta_2}{0.7/0.8}, \frac{\theta_3}{1}, \frac{\theta_4}{0.06/0.8}\right\}=\left\{\frac{\theta_1}{0.1875}, \frac{\theta_2}{0.875}, \frac{\theta_3}{1}, \frac{\theta_4}{0.075}\right\}.
\]
The consonant mass function $m_{\tF\varodot\tG}$ corresponding to $\pi_{\tF\varodot\tG}$ is
\begin{align*}
m_{\tF\varodot\tG}(\{\theta_3\}) &= 1-0.875=0.125\\
m_{\tF\varodot\tG}(\{\theta_2,\theta_3\}) &= 0.875-0.1875=0.6875\\
m_{\tF\varodot\tG}(\{\theta_1,\theta_2,\theta_3\}) &= 0.1875-0.075=0.1125\\
m_{\tF\varodot\tG}(\Theta) &=0.075.
\end{align*}
Consider now another situation in which we receive two pieces of evidence represented by the following consonant mass functions:
\begin{align*}
m_\tF(\{\theta_2\}) &= 1-0.8=0.2\\
m_\tF(\{\theta_2,\theta_3\}) &= 0.8-0.5=0.3\\
m_\tF(\{\theta_1,\theta_2,\theta_3\}) &= 0.5-0.3=0.2\\
m_\tF(\Theta) &=0.3,
\end{align*}
and
\begin{align*}
m_\tG(\{\theta_3\}) &= 1-0.7=0.3\\
m_\tG(\{\theta_2,\theta_3\}) &= 0.7-0.3=0.4\\
m_\tG(\{\theta_1,\theta_2,\theta_3\}) &= 0.3-0.2=0.1\\
m_\tG(\Theta) &=0.2.
\end{align*}
Mass functions $m_\tF$ and $m_\tG$ induce the same belief functions as, respectively, $\tF$ and $\tG$. Yet, they are combined differently by Dempster's rule. Their orthogonal sum $m_\tF\oplus m_\tG$ is
\begin{align*}
(m_\tF\oplus m_\tG)(\{\theta_3\}) &= 0.24/(1-0.06)\approx 0.255\\
(m_\tF\oplus m_\tG)(\{\theta_2\}) &= 0.14/(1-0.06)\approx0.149\\
(m_\tF\oplus m_\tG)(\{\theta_2,\theta_3\}) &= 0.41/(1-0.06)\approx 0.436\\
(m_\tF\oplus m_\tG)(\{\theta_1,\theta_2,\theta_3\}) &=  0.09/(1-0.06)\approx 0.0957\\
(m_\tF\oplus m_\tG)(\Theta) &=0.06/(1-0.06)\approx 0.0638,
\end{align*}
which is different from $m_{\tF\varodot\tG}$. In particular, $m_\tF\oplus m_\tG$ is not consonant. Its contour function of $m_\tF\oplus m_\tG$ is
\begin{align*}
pl_{m_\tF\oplus m_\tG}(\theta_1)&=(0.09+0.06)/(1-0.06)\approx 0.160\\
pl_{m_\tF\oplus m_\tG}(\theta_2)&=(0.14+0.41+0.09+0.06)/(1-0.06)\approx 0.745\\
pl_{m_\tF\oplus m_\tG}(\theta_3)&=(0.24+0.41+0.09+0.06)/(1-0.06)\approx 0.851\\
pl_{m_\tF\oplus m_\tG}(\theta_4)&=0.06/(1-0.06)\approx 0.0638.
\end{align*}
It can be checked that $\pi_{\tF\varodot\tG}$ and $pl_{m_\tF\oplus m_\tG}$ are proportional, with $\pi_{\tF\varodot\tG}/pl_{m_\tF\oplus m_\tG}=1.175$.

\end{Ex}

The above considerations show that it is important, in practice, to determine whether a piece of evidence should be represented by a possibility distribution or by a consonant mass function. It seems reasonable to use the former representation for  reliable but fuzzy evidence  such as conveyed, e.g., by natural language, as in the proposition ``John is tall''. More surprisingly, it appears that \emph{statistical evidence} should also be represented in that way, as will be shown in Section \ref{sec:stat}. In contrast, consonant evidence  usually arises when combining several elementary pieces of evidence, such as simple mass functions of the form $m_i(F_i)=p_i$, $m_i(\Theta)=1-p_i$ with $F_1\subseteq\ldots\subseteq F_f$. Such simple mass functions may be obtained by combining expert opinions or sensor readings telling us that $F_i\subseteq \Theta$ with confidence degree $p_i$.

\paragraph{Possibility and necessity of a fuzzy event}

In  \cite{zadeh78}, Zadeh  defines the possibility and the necessity of a fuzzy event $\tA \in \calF(\Theta)$ as
\begin{equation}
\label{eq:poss_fuzzy}
\Pi_\tF^{(S)}(\tA):= \max_{\theta\in\Theta} \; (\tA\wedge \tF)(\theta) = h(\tA \wedge \tF).
\end{equation}
and
\begin{subequations}
\label{eq:nec_fuzzy}
\begin{align}
\label{eq:NfromPoss}
N^{(S)}_\tF(\tA) &:= 1-\Pi^{(S)}_\tF(\tA^c)\\
 &= 1-\max_{\theta\in\Theta} \; [1-A(\theta)] \wedge \tF(\theta)\\
  &= \min_{\theta\in\Theta} \left\{1- [1-\tA(\theta)] \wedge \tF(\theta)\right\}\\
&= \min_{\theta\in\Theta} \tA(\theta) \vee [1-\tF(\theta)]\\
&= \min_{\theta\in\Theta} \; (\tA \vee \tF^c)(\theta).
\end{align}
\end{subequations}
We note that $\Pi_\tF^{(S)}(\tA)$ and $N_\tF^{(S)}(\tA)$ are Sugeno integrals of the mapping $\tA: \Theta \to [0,1]$ with respect to $\Pi_\tF$ and $N_\tF$, respectively \cite[Section 7.6.2]{dubois00a}. We can also remark that $\Int(\tA,\tF):=h(\tA \wedge \tF)$ can be seen as a \emph{degree of intersection} of $\tA$ and $\tF$, whereas $\Incl(\tF,\tA):=\min_{\theta\in\Theta} (\tA\vee \tF^c)(\theta)$ can be seen as a \emph{degree of inclusion} of $\tF$ in $\tA$.

It is easy to check that $\Pi_\tF^{(S)}$ and $N_\tF^{(S)}$ are still, respectively, possibility and necessity measures in the lattice $(\calF(\Theta),\wedge,\vee)$, as 
\begin{subequations}
\label{eq:Poss_max_fuzzy}
\begin{equation}
\label{eq:maxitive_fuzzy}
\Pi_\tF^{(S)}(\tA\vee \tB)=\Pi_\tF^{(S)}(\tA) \vee \Pi_\tF^{(S)}(\tB)
\end{equation}
 and 
\begin{equation}
N_\tF^{(S)}(\tA\wedge \tB)=N_\tF^{(S)}(\tA) \wedge N_\tF^{(S)}(\tB)
\end{equation} 
\end{subequations}
for all fuzzy sets $\tA$ and $\tB$, which generalize \eqref{eq:Boolmax}-\eqref{eq:Boolmin}. It is also easy to show that $N_\tF^{(S)}(\tA) \le \Pi_\tF^{(S)}(\tA)$ for all $\tA \in \calF(\Theta)$ \cite{dubois85a}. Dubois and Prade \cite{dubois85a} also show that $N_\tF^{(S)}$ is a belief function in the lattice $(\calF(\Theta),\wedge,\vee)$, and $\Pi_\tF^{(S)}$ is the dual plausibility function.

Dubois {\em et al.} review generalizations of \eqref{eq:poss_fuzzy}-\eqref{eq:nec_fuzzy} as well as alternative definitions  in \cite[Section 7.6]{dubois00a}. In particular, they remark that we can  identify $\Pi_\tF$ and $N_\tF$ to, respectively,  the plausibility function $Pl_{m_\tF}$ and the belief function $Bel_{m_\tF}$ induced by the  consonant mass function $m_\tF$  defined by \eqref{eq:consonant}. We can then define the possibility and necessity of a fuzzy event $\tA$ from \eqref{eq:fuzzyBel} and \eqref{eq:BelPlChoquet} by the Choquet integrals \new{of the mapping $\tilde A: \Theta \to [0,1]$ with respect to the non-additive functions $\Pi_\tF$ and $N_\tF$}:
\begin{subequations}
\label{eq:poss_nec_fuzzy_Choquet}
\begin{align}
\Pi_\tF^{(C)}(\tA) &:= \int_0^1 \Pi_\tF(\cut\alpha\tA)d\alpha=\sum_{A\subseteq \Theta} m_\tF(A)\max_{\theta\in A} \tA(\theta)=Pl_{m_\tF}(\tA)\\
\label{eq:nec_fuzzy_Choquet}
N_\tF^{(C)}(\tA) &:= \int_0^1 N_\tF(\cut\alpha\tA)d\alpha=\sum_{A\subseteq \Theta} m_\tF(A)\min_{\theta\in A} \tA(\theta)=Bel_{m_\tF}(\tA).
\end{align}
\end{subequations}
From \eqref{eq:poss_nec1}, we then have
\begin{subequations}
\label{eq:poss_nec_fuzzy_Choquet1}
\begin{equation}
\Pi_\tF^{(C)}(\tA) = \int_0^1 \left( \int_0^1\Pi_{\cut\beta\tF}(\cut\alpha\tA)d\beta \right)d\alpha= \int_0^1 \Pi_{\cut\beta\tF}(\tA)d\beta
\end{equation}
and
\begin{equation}
N_\tF^{(C)}(\tA) = \int_0^1 \left( \int_0^1 N_{\cut\beta\tF}(\cut\alpha\tA)d\beta \right)d\alpha= \int_0^1 N_{\cut\beta\tF}(\tA)d\beta,
\end{equation}
with
\begin{equation}
 \Pi_{\cut\beta\tF}(\tA)=\int_0^1 \Pi_{\cut\beta\tF}(\cut\alpha\tA)d\alpha=\max_{\theta\in \cut\beta\tF} \tA(\theta)=\max_{\{\theta: \tF(\theta)\ge \beta\}} \tA(\theta)
\end{equation}
and
\begin{equation}
N_{\cut\beta\tF}(\tA)=\int_0^1 N_{\cut\beta\tF}(\cut\alpha\tA)d\alpha=\min_{\theta\in \cut\beta\tF} \tA(\theta)=\min_{\{\theta: \tF(\theta)\ge \beta\}} \tA(\theta).
\end{equation}
\end{subequations}
As $N_\tF^{(S)}$, function $N_\tF^{(C)}$ defined by \eqref{eq:nec_fuzzy_Choquet} is a  belief function on the lattice $(\calF(\Theta),\wedge,\vee)$, and $\Pi_\tF^{(C)}$ is its dual plausibility function. However,  $\Pi_\tF^{(C)}$ and $N_\tF^{(C)}$ are no longer possibility and necessity measures, as they fail to satisfy the basic axioms \eqref{eq:Poss_max_fuzzy} \cite{dubois85a}.



The following example illustrates the difference between $\Pi_\tF^{(S)}$ and $\Pi_\tF^{(C)}$. 
\begin{Ex}
\label{ex:comb_poss}
Consider the fuzzy sets $\tF$ and $\tG$ of Example \ref{ex:consonant}. We have
\[
\Pi_\tF^{(S)}(\tG)=h(\tF\wedge\tG)=h\left(\left\{\frac{\theta_1}{0.3}, \frac{\theta_2}{0.7}, \frac{\theta_3}{0.8}, \frac{\theta_4}{0.2}\right\}\right)=0.8
\]
and
\[
N_\tF^{(S)}(\tG)=1-h(\tF\wedge\tG^c)=1-h\left(\left\{\frac{\theta_1}{0.5}, \frac{\theta_2}{0.3}, \frac{\theta_3}{0}, \frac{\theta_4}{0.3}\right\}\right)=1-0.5=0.5,
\]
but
\[
\Pi_\tF^{(C)}(\tG)=0.2\times 0.7 + 0.3\times 1 + 0.2\times 1+ 0.3\times 1=0.94.
\]
and
\[
N_\tF^{(C)}(\tG)=0.2\times 0.7 + 0.3\times 0.7 + 0.2\times 0.3+ 0.3\times 0.2=0.47.
\]
\end{Ex}
 
To conclude this section, we can remark that the guaranteed possibility function \eqref{eq:GP} can be extended to fuzzy events as well. It can easily be seen that, for crisp $A$,
\[
\min_{\theta \in A} \tF(\theta)=\min_{\theta\in\Theta} \; (A^c \vee \tF)(\theta)=\Incl(A,\tF).
\]
A natural definition for the guaranteed possibility of fuzzy event $\tA$, in the spirit of Zadeh's definitions \eqref{eq:poss_fuzzy}-\eqref{eq:nec_fuzzy} for the possibility and belief of a fuzzy event is, thus,
\begin{equation}
\label{eq:fGP}
\Delta_\tF^{(S)}(\tA):= \Incl(\tA,\tF)=\min_{\theta \in \Theta} \; (\tA^c \vee \tF)(\theta).
\end{equation}
As a generalization of \eqref{eq:Delta_min}, we have
\begin{equation}
\label{eq;Deltamin}
\Delta_\tF^{(S)}(\tA\vee\tB) = \Delta_\tF^{(S)}(\tA) \wedge \Delta_\tF^{(S)}(\tB).
\end{equation}
Alternatively, we can define the guaranteed possibility of fuzzy event $\tA$ from \eqref{eq:mQfuzzy} by the Choquet integral
\begin{equation}
\label{eq:fGP1}
\Delta_\tF^{(C)}(\tA):= \int_0^1 \Delta_\tF(\cut\alpha\tA)d\alpha= \sum_{A\subseteq \Theta} m_\tF(A)\left(1-\max_{\theta\not\in A} \tA(\theta)\right)
 d\alpha=Q_{m_\tF}(\tA).
\end{equation}
Unfortunately,  equality \eqref{eq;Deltamin} does not hold anymore with this alternative definition.

\begin{Ex}
Consider again the fuzzy sets $\tF$ and $\tG$ of Examples \ref{ex:consonant} and \ref{ex:comb_poss}. From
\[
\tG^c\vee\tF=\left\{\frac{\theta_1}{0.7}, \frac{\theta_2}{1}, \frac{\theta_3}{0.8}, \frac{\theta_4}{0.8}\right\},
\]
we have $\Delta_\tF^{(S)}(\tG)=0.7$, but
\[
\Delta_\tF^{(C)}(\tG)=0.2(1-1)+ 0.3(1-0.3)+0.2(1-0.2)+0.3=0.67.
\]
Now, let
\[
\tH=\left\{\frac{\theta_1}{1}, \frac{\theta_2}{0.6}, \frac{\theta_3}{0.}, \frac{\theta_4}{0.1}\right\}.
\]
We have
\[
\tH^c\vee\tF=\left\{\frac{\theta_1}{0.5}, \frac{\theta_2}{1}, \frac{\theta_3}{0.8}, \frac{\theta_4}{0.9}\right\},
\]
hence $\Delta_\tF^{(S)}(\tH)=0.5$, and
\[
(\tG\vee\tH)^c\vee\tF=\left\{\frac{\theta_1}{0.5}, \frac{\theta_2}{1}, \frac{\theta_3}{0.8}, \frac{\theta_4}{0.8}\right\},
\]
hence $\Delta_\tF^{(S)}(\tG\vee \tH)=0.5=\Delta_\tF^{(S)}(\tG) \wedge \Delta_\tF^{(S)}(\tH)$. But we have
\[
\Delta_\tF^{(C)}(\tH)=0.2(1-1)+ 0.3(1-1)+0.2(1-0.1)+0.3=0.48
\]
and
\begin{align*}
\Delta_\tF^{(C)}(\tG\vee \tH)&=0.2(1-1)+ 0.3(1-1)+0.2(1-0.2)+0.3=0.46\\
&  \neq \Delta_\tF^{(C)}(\tG) \wedge \Delta_\tF^{(C)}(\tH).
\end{align*}

\end{Ex}

\section{Uncertain and fuzzy information: fuzzy mass functions}
\label{sec:fuzzyDS}

To handle cases where one consonant belief function is based on fuzzy evidence and the other one on uncertain evidence, or to handle evidence that is \emph{both} uncertain and fuzzy, we need to generalize DS and possibility theories. Such a generalization  is exposed in this section. We first review the notion of fuzzy mass function (Section \ref{subsec:FBS}). In Section \ref{subsec:combFBS},  we propose a combination operator that extends both Dempster's rule \eqref{eq:dempster} and the normalized product \eqref{eq:norm_prod}; we also define extensions of the disjunctive operators reviewed in Section \ref{subsec:DS} and \ref{subsec:poss}. The degrees of belief and plausibility of fuzzy events are then defined in Section \ref{subsec:belplfuzzy}. 

\subsection{Fuzzy mass functions}
\label{subsec:FBS}

Following \cite{zadeh79}, let us  assume that we receive uncertain and fuzzy information, which can be modeled as follows. As in Section \ref{subsec:DS}, we assume that the  evidence can be interpreted in different ways, and the set of interpretations is denoted by $\Omega$. If interpretation $\omega\in \Omega$ holds, then we know for sure that  proposition ``$\btheta \is \tGamma(\omega)$'' is true, where $\tGamma(\omega)$ is a normal fuzzy subset of $\Theta$. Denoting by $\calF^*(\Theta)$ the set of all normal fuzzy subsets of $\Theta$, we thus have a mapping $\tGamma: \Omega \rightarrow \calF^*(\Theta)$. If, as before, we assume the existence of a probability measure $P$ on $(\Omega,2^\Omega)$, then the tuple $(\Omega,2^\Omega,P,\tGamma)$ is a \emph{random fuzzy set} \cite{feron76,kwakernaak78,kwakernaak79} (also called a ``fuzzy random variable'' when  $\Theta$ is $\reels^p$).  

Let $\tm$ be the mapping from $\calF^*(\Theta)$ to $[0,1]$ defined as
\[
\tm(\tF):=P(\{\omega\in \Omega : \tGamma(\omega)=\tF\}).
\]
Because $\Omega$ is assumed to be finite, there is only a finite number of fuzzy subsets $\tF$ such that $\tm(\tF)>0$,   called the (fuzzy) focal sets of $m$. The set of focal sets is denoted as $\F(m)=\{\tF_1,\ldots,\tF_f\}$. We also use the notation $m_i:=m(\tF_i)$, $i=1,\ldots,n$. Mapping $\tm$ is called a \emph{fuzzy mass function}\footnote{This notion should not be confused with that of \emph{fuzzy-valued mass function} introduced in \cite{denoeux00c}. A fuzz-valued mass function assigns fuzzy numbers to crisp focal sets, and can be interpreted as a fuzzy set of crisp mass functions. We could, of course, ``fuzzify'' both the masses and the focal sets; such a generalization will not be considered in this paper.}. The notion of fuzzy mass function extends that of DS mass function  recalled in Section \ref{subsec:DS}. The number $\tm(\tF_i)$ is interpreted as the degree with which the evidence supports  the proposition $\btheta \textrm{ is } \tF_i$, without supporting any more specific proposition. 

If interpretation $\omega$ holds, we know that ``$\btheta \is \tGamma(\omega)$''. The  possibility and necessity of a subset $A\subseteq \Theta$ are, respectively, $\Pi_{\tGamma(\omega)}(A)$ and $N_{\tGamma(\omega)}(A)$ defined by \eqref{eq:poss} and \eqref{eq:nec}. As we only know that interpretation $\omega$ holds with probability $P(\{\omega\})$, we can compute the expected possibility and the expected necessity \cite{zadeh79} as
\begin{equation}
\label{eq:defPl1}
Pl_\tm(A) = \sum_{\omega\in \Omega} P(\{\omega\}) \Pi_{\tGamma(\omega)}(A) = \sum_{i=1}^f m_i \Pi_{\tF_i}(A)=  \sum_{i=1}^f m_i \max_{\theta\in A} \tF_i(\theta)
\end{equation}
and
\begin{equation}
\label{eq:defBel1}
Bel_\tm(A) = \sum_{\omega\in \Omega} P(\{\omega\}) N_{\tGamma(\omega)}(A) =\sum_{i=1}^f m_i N_{\tF_i}(A)=\sum_{i=1}^f m_i \min_{\theta\not\in A}[1- \tF_i(\theta)].
\end{equation}

Functions $Pl_\tm$ and $Bel_\tm$ are, respectively, mixtures of possibility and necessity measures. As we have seen that each necessity measure $N_{\tF_i}$ is a belief function, and the set of belief functions is convex, $Bel_\tm$ is still a belief function, and $Pl_\tm$ (which verifies $Pl_\tm(A)=1-Bel_\tm(A^c)$ for all $A\subseteq\Theta$) is the dual plausibility function. The contour function of $\tm$ is equal to the mean of the membership functions of its focal sets:
\begin{equation}
\label{eq:contour_fuzzy}
pl_\tm(\theta)=\sum_{i=1}^f m_i \tF_i(\theta).
\end{equation}

We can also define the commonality of $A$ as its expected guaranteed possibility from \eqref{eq:GP}:
\begin{equation}
\label{eq:deQ1}
Q_\tm(A) = \sum_{\omega\in \Omega} P(\{\omega\}) \Delta_{\tGamma(\omega)}(A) =\sum_{i=1}^f m_i \Delta_{\tF_i}(A)=\sum_{i=1}^f m_i \min_{\theta\in A} \tF_i(\theta).
\end{equation}

A fuzzy mass function $\tm$  with focal sets $\tF_1,\ldots,\tF_f$ and masses $m_1,\ldots,m_f$ can also be described as a collection of standard (crisp) mass functions $\cut\alpha{\tm}$ with focal sets $\cut\alpha\tF_1,\ldots,\cut\alpha\tF_f$ and the same masses $m_1,\ldots,m_f$. Each mass function $\cut\alpha\tm$ is induced by the random set $(\Omega,2^\Omega,P,\cut\alpha\tGamma)$, with $\cut\alpha\tGamma(\omega)=\cut\alpha{[\tGamma(\omega)]}$. From \eqref{eq:poss_nec1} and \eqref{eq:GP}, we have
\begin{subequations}
\label{eq:alphacut}
\begin{align}
\label{eq:alphacutbel}
Bel_\tm(A) &= \int_0^1 Bel_{\cut\alpha\tm}(A) d\alpha,\\
\label{eq:alphacutpl}
Pl_\tm(A) &= \int_0^1 Pl_{\cut\alpha\tm}(A) d\alpha,
\end{align}
and
\begin{equation}
Q_\tm(A)=  \int_0^1 Q_{\cut\alpha\tm}(A) d\alpha
\end{equation}
\end{subequations}
for all $A\subseteq \Theta$. \new{We note that  equalities \eqref{eq:alphacutbel} and \eqref{eq:alphacutpl} are proved in \cite{couso11} in a more general setting.}

It is important to remark that the concept of fuzzy mass functions allows us to generalize both DS theory and possibility theory. In particular, a \emph{logical} fuzzy mass function $\tm$ with a single fuzzy focal set $\tF$ is equivalent to a possibility distribution $\pi=\tF$.

\begin{Rem}
In ``classical'' DS theory there is a one-to-one correspondence between (crisp) mass functions and belief functions. The fundamental reason is that, for a given belief function $Bel$, the system of linear equations $Bel(A)=\sum_{B\subseteq A} m(B)$ for all $\emptyset\neq A\subseteq\Omega$ has a unique solution. In the ``fuzzified'' version of DS theory, there is a many-to-one correspondence between fuzzy mass functions and belief functions, i.e., a given belief function can be obtained from  many fuzzy mass functions. The simplest example is that of a belief function verifying $Bel(A\cap B)=Bel(A)\wedge Bel(B)$ for all $A,B\subseteq\Theta$ (i.e., a necessity measure), which, as shown in Section \ref{subsec:poss}, can be induced both by a crisp consonant mass function and by a logical fuzzy mass function.
\end{Rem}

In the following section, we define a combination operator for fuzzy mass functions that extends both the normalized product of possibility distributions \eqref{eq:norm_prod} and Dempster's rule \eqref{eq:dempster}. 

\subsection{Combination of fuzzy mass functions}
\label{subsec:combFBS}

Let us now assume that we have two fuzzy mass functions $\tm_1$ and $\tm_2$ 
induced by independent random fuzzy sets $(\Omega_1,2^{\Omega_1},P_1,\tGamma_1)$ and $(\Omega_2,2^{\Omega_2},P_2,\tGamma_2)$. If interpretations $\omega_1\in \Omega_1$ and $\omega_2\in \Omega_2$ both hold, then we can infer the fuzzy proposition ``$\btheta \is \tGamma_{\cap_\top}(\omega_1,\omega_2)$'', with $\tGamma_{\cap_\top}(\omega_1,\omega_2):=\tGamma_1(\omega_1)\cap_\top \tGamma_2(\omega_2)$, where $\cap_\top$ is a fuzzy intersection operator based on a t-norm $\top$. However, the random fuzzy set $(\Omega_1\times\Omega_2, 2^{\Omega_1\times\Omega_2},P_1\otimes P_2,\tGamma_{\cap_\top})$ may not be normalized, as some fuzzy sets $\tGamma_{\cap_\top}(\omega_1,\omega_2)$ may not be normal. To obtain a normal random fuzzy set,  we can replace $\tGamma_{\cap_\top}$ by the mapping
\begin{equation}
\label{eq:Gamma_norm}
\tGamma_{\barcap_\top}(\omega_1,\omega_2):=\tGamma_1(\omega_1)\barcap_\top \tGamma_2(\omega_2)=\frac{\tGamma_1(\omega_1)\cap_\top \tGamma_2(\omega_2)}{h(\tGamma_1(\omega_1)\cap_\top \tGamma_2(\omega_2))},
\end{equation}
where $\barcap_\top$ denotes normalized $\top$-intersection, and we can condition $P_1\otimes P_2$ by the following fuzzy subset $\tTheta^*$ of $\Omega_1\times\Omega_2$, which is a natural generalization of $\Theta^*$ in \eqref{eq:Theta_star}:
\begin{equation}
\tTheta^*(\omega_1,\omega_2)=h(\tGamma_{\cap_\top}(\omega_1,\omega_2)), \quad \forall(\omega_1,\omega_2)\in \Omega_1\times\Omega_2.
\end{equation}
Using Zadeh's definition of the probability of a fuzzy event \eqref{eq:probaFuzzy}, the conditional probability $(P_1\otimes P_2)(\cdot \mid \tTheta^*)$ is
\begin{align*}
(P_1\otimes P_2)(B \mid \tTheta^*)&= \frac{(P_1\otimes P_2)(B \cap \tTheta^*)}{(P_1\otimes P_2)(\tTheta^*)}\\
&=\frac{\sum_{(\omega_1,\omega_2)\in B} P_1(\omega_1)P_2(\omega_2) h(\tGamma_{\cap_\top}(\omega_1,\omega_2))}{\sum_{(\omega_1,\omega_2)\in \Omega_1\times\Omega_2} P_1(\omega_1)P_2(\omega_2) h(\tGamma_{\cap_\top}(\omega_1,\omega_2))},
\end{align*}
for all $B\subseteq \Omega_1\times\Omega_2$. The fuzzy mass function generated by the random fuzzy set $(\Omega_1\times\Omega_2, 2^{\Omega_1\times\Omega_2},(P_1\otimes P_2)(\cdot\mid\tTheta^*),\tGamma'_{\cap_\top})$ is then
\begin{equation}
\label{eq:dempster_fuzzy1}
(\tm_1\Cap_\top \tm_2)(\tF):=\frac{ \sum_{\tG\barcap_\top \tH=\tF} h(\tG \cap_\top \tH) \tm_1(\tG)\tm_2(\tH)}{\sum_{(\tG,\tH)\in \F(\tm_1)\times \F(\tm_2)} h(\tG\cap_\top \tH) \tm_1(\tG)\tm_2(\tH)}
\end{equation}
for all $\tF\in \calF^*(\Theta)$. This operation was proposed by Yen \cite{yen90}, using $\top=\min$. The normalization operation in \eqref{eq:dempster_fuzzy1} was called  ``soft normalization'' by Yager \cite{yager96}. 

In general, the operation $\Cap_\top$ defined by \eqref{eq:dempster_fuzzy1}  is not associative. However, it is associative if we use the \emph{product t-norm}  for defining the fuzzy set intersection. In the following, we will thus  define the \emph{orthogonal sum} $\tm_1\oplus\tm_2$ of two fuzzy mass functions $\tm_1$ and $\tm_2$ (where the symbol $\oplus$ has the same meaning as $\Cap_\top$ with $\top=\text{product}$) as
\begin{equation}
\label{eq:dempster_fuzzy}
(\tm_1\oplus \tm_2)(\tF):=\frac{ \sum_{\tG \varodot\tH=\tF} h(\tG \cdot\tH) \tm_1(\tG)\tm_2(\tH)}{\sum_{(\tG,\tH)\in \F(\tm_1)\times \F(\tm_2)} h(\tG\cdot \tH) \tm_1(\tG)\tm_2(\tH)},
\end{equation}
for all $\tF\in \calF^*(\Theta)$. The operation defined by \eqref{eq:dempster_fuzzy} will be called the \emph{generalized product-intersection} rule. The denominator of the right-hand side of \eqref{eq:dempster_fuzzy} can be denoted as $1-\kappa$, where $\kappa$ can be interpreted as the \emph{degree of conflict} between fuzzy mass functions $\tm_1$ and $\tm_2$. It is clear that \eqref{eq:dempster_fuzzy} is a proper generalization of \eqref{eq:dempster}, i.e., it gives the same result when the focal sets of $\tm_1$ and $\tm_2$ are crisp. It also boils down to normalized intersection \eqref{eq:norm_prod} when combining two logical fuzzy mass functions $\tm_1$ and $\tm_2$ such that $\tm_1(\tF)=1$ and $\tm_2(\tG)=1$. As a consequence of Proposition  \ref{prop:product_intersection}, the generalized product-intersection rule is associative, as stated in the following proposition.

\begin{Prop}
\label{prop:dempster_assoc}
Let $m_1$, $m_2$ and $m_3$ be fuzzy mass functions of $\Theta$, and let $\oplus$ denote the generalized product-intersection  defined by \eqref{eq:dempster_fuzzy}. Then,
\[
(m_1 \oplus m_2) \oplus m_3 = m_1 \oplus (m_2\oplus m_3).
\]
\end{Prop}

\begin{proof}
We have
\begin{align*}
[(m_1 \oplus m_2) \oplus m_3](\tA)&= \frac{\sum_{\tH\varodot \tK=\tA} (m_1\oplus m_2)(\tH) m_3(\tK) h(\tH\cdot\tK)}{\sum_{\tH, \tK} (m_1\oplus m_2)(\tH) m_3(\tK) h(\tH\cdot \tK)}\\
&= \frac{\sum_{\tH\varodot \tK=\tA} \left[
\frac{\sum_{\tF\varodot \tG=\tH} m_1(\tF)m_2(\tG)h(\tF\cdot \tG)}{\sum_{\tF,\tG} m_1(\tF)m_2(\tG)} \right]
m_3(\tK) h(\tH\cdot\tK)}
{\sum_{\tH, \tK} \left[
\frac{\sum_{\tF\varodot \tG=\tH} m_1(\tF)m_2(\tG)h(\tF\cdot\tG)}{\sum_{\tF,\tG} m_1(\tF)m_2(\tG) } \right] m_3(\tK) h(\tH\tK)}\\
&=\frac{\sum_{\tF\varodot \tG \varodot \tK=\tA} m_1(\tF)m_2(\tG)m_3(\tH) h(\tF\cdot \tG)h((\tF\varodot \tG)\cdot \tK)}{\sum_{\tF, \tG, \tK} m_1(\tF)m_2(\tG)m_3(\tH) h(\tF\cdot \tG)h((\tF\varodot \tG)\cdot \tK)}.
\end{align*}
Now, 
\[
h(\tF\cdot \tG)h((\tF\varodot \tG)\cdot \tK)=h(\tF\cdot\tG)h\left(\frac{\tF\cdot\tG}{h(\tF\cdot\tG)}\cdot \tK\right)=h(\tF\cdot\tG\cdot\tK).
\]
Hence,
\begin{equation}
\label{eq:proof_dempster}
[(m_1 \oplus m_2) \oplus m_3](\tA)=\frac{\sum_{\tF\varodot \tG \varodot \tK=\tA} m_1(\tF)m_2(\tG)m_3(\tH) h(\tF\cdot\tG\cdot\tK)}{\sum_{\tF, \tG, \tK} m_1(\tF)m_2(\tG)m_3(\tH) h(\tF\cdot\tG\cdot\tK)}.
\end{equation}
Consequently, we can permute the indices in \eqref{eq:proof_dempster} and write
\[
(m_1 \oplus m_2) \oplus m_3=(m_2 \oplus m_3) \oplus m_1=m_1 \oplus (m_2 \oplus m_3).
\]
\end{proof}

\begin{Ex}
Continuing Example \ref{ex:consonant}, assume that we have the following two fuzzy mass functions:
\[
m_1(\tF):=0.6, \quad m_1(\Theta):=0.4
\] 
and
\[
m_2(\tG):=0.7, \quad m_2(\Theta):=0.3,
\] 
where $\tF$ and $\tG$ are defined as in Example \ref{ex:consonant}. As $h(\tF\cdot\tG)=0.8$, we have
\begin{align*}
(m_1\oplus m_2)(\tF\varodot\tG)&=\frac{0.6\times 0.7\times 0.8}{0.6\times 0.7\times 0.8 + 0.6\times 0.3 + 0.4\times 0.7 + 0.4\times 0.3}\approx 0.37\\
(m_1\oplus m_2)(\tF)&=\frac{0.6\times 0.3}{0.6\times 0.7\times 0.8 + 0.6\times 0.3 + 0.4\times 0.7 + 0.4\times 0.3}\approx 0.20\\
(m_1\oplus m_2)(\tG)&=\frac{0.7\times 0.4}{0.6\times 0.7\times 0.8 + 0.6\times 0.3 + 0.4\times 0.7 + 0.4\times 0.3}\approx 0.31\\
(m_1\oplus m_2)(\Theta)&=\frac{0.3\times 0.4}{0.6\times 0.7\times 0.8 + 0.6\times 0.3 + 0.4\times 0.7 + 0.4\times. 0.3}\approx 0.13.
\end{align*}
\end{Ex}

The following propositions generalize Propositions \ref{prop:combQ} and \ref{prop:combBayes}.

\begin{Prop}
\label{prop:contour_fuzzy}
Let $\tm_1$ and $\tm_2$ be two fuzzy mass functions with contour functions $pl_{\tm_1}$ and $pl_{\tm_2}$. Then, the contour function of $\tm_1\oplus\tm_2$ is given by
\[
pl_{\tm_1\oplus\tm_2}=\frac{pl_{\tm_1}pl_{\tm_2}}{1-\kappa}
\]
with
\[
\kappa=1-\sum_{(\tG,\tH)\in \F(\tm_1)\times \F(\tm_2)} h(\tG\cdot \tH) \tm_1(\tG)\tm_2(\tH).
\]
\end{Prop}
\begin{proof}
Let $\tF_1,\ldots,\tF_{f_1}$ and $\tG_1,\ldots,\tG_{f_2}$ be the focal sets of $\tm_1$ and $\tm_2$, respectively, with corresponding masses $m_{11},\ldots,m_{1f_1}$ and $m_{21},\ldots,m_{2f_2}$. From \eqref{eq:contour_fuzzy},
\begin{align*}
pl_{\tm_1\oplus\tm_2}(\theta)&=\sum_{i,j} \frac{m_{1i}m_{2j} h(\tF_i\tG_j)}{1-\kappa} (\tF_i\varodot\tG_j)(\theta)\\
&=\frac{\sum_{i,j} m_{1i}m_{2j}  \tF_i(\theta)\tG_j(\theta)}{1-\kappa}\\
&=\frac{\left(\sum_{i} m_{1i} \tF_i(\theta)\right) \left(\sum_{j} m_{2j} \tG_j(\theta)\right)}{1-\kappa}=\frac{pl_{\tm_1}(\theta)pl_{\tm_2}(\theta)}{1-\kappa}.
\end{align*}
\end{proof}

\begin{Prop}
\label{prop:combBayesfuzzy}
Let  $\tm_1$ be a fuzzy mass function and $m_2$ a Bayesian mass function. Then the orthogonal sum $\tm_1\oplus m_2$ is Bayesian and it is given by
\[
(\tm_1\oplus m_2)(\{\theta\})=\frac{\sum_{i=1}^f \tF_i(\theta)\tm_1(\tF_i) m_2(\{\theta\})}{\sum_{\theta'\in\Theta}\sum_{i=1}^f \tF_i(\theta')\tm_1(\tF_i) m_2(\{\theta'\})}=\frac{pl_{\tm_1}(\theta)m_2(\{\theta\})}{\sum_{\theta'\in\Theta}pl_{\tm_1}(\theta')m_2(\{\theta'\})}.
\]
\end{Prop}
\begin{proof}
As the normalized product of each focal set $\tF_i$ of $\tm_1$ with each focal set $\{\theta\}$ of $m_2$ equals $\{\theta\}$, the orthogonal $\tm_1\oplus m_2$ is Bayesian. The expression of $(\tm_1\oplus m_2)(\{\theta\})$ follows directly from Proposition \ref{prop:contour_fuzzy}.
\end{proof}

In DS theory, it is well known that Dempster's rule extends Bayesian conditioning, i.e., the orthogonal sum of a Bayesian belief function $P$ and a logical belief function focussed on some subset $A$ is a Bayesian belief function that is identical to the conditional probability measure $P(\cdot\mid A)$ \cite{shafer76}. In a similar way, the  generalized product-intersection rule \eqref{eq:dempster_fuzzy} extends conditioning of a probability measure by a fuzzy event, as stated in the following proposition.

\begin{Prop}
Let $m$ be a Bayesian mass function with corresponding probability measure $P$ and $\tm$ a logical fuzzy mass function with focal set $\tA \in \calF(\Theta)$. Then $m\oplus \tm$ is Bayesian and the corresponding belief function $Bel_{m\oplus \tm}$ is identical to the probability measure  $P(\cdot \mid \tA)$ obtained by conditioning $P$ by fuzzy event $\tA$.
\end{Prop}
\begin{proof}
From Proposition \ref{prop:combBayesfuzzy}, $m\oplus \tm$ is Bayesian and it is given by
\[
(m\oplus \tm)(\{\theta\})=\frac{m(\{\theta\}) \tA(\theta)}{\sum_{\theta' \in\Theta}m(\{\theta' \}) \tA(\theta')}=\frac{P(\{\theta\}) \tA(\theta)}{\sum_{\theta' \in\Theta}P(\{\theta' \}) \tA(\theta')}.
\]
Consequently, for any $A\subseteq \Theta$,
\[
Bel_{m\oplus \tm}(A)=\sum_{\theta \in A} (m\oplus \tm)(\{\theta\})=\frac{ \sum_{\theta \in A} P(\{\theta\}) \tA(\theta)}{\sum_{\theta' \in\Theta}P(\{\theta' \}) \tA(\theta')} = P(A \mid \tA).
\]
\end{proof}

Finally, we can remark that the disjunctive rule \eqref{eq:disj_rule} can also be extended to fuzzy mass functions. The extension is simpler in this case, because there is no normalization. For any t-conorm $\bot$, we can define a disjunctive operator $\Cup_\bot$ as
\begin{equation}
\label{eq:disj_fuzzy}
(\tm_1\Cup_\bot \tm_2)(\tF):= \sum_{\tG\cup_\bot\tH=\tF}  \tm_1(\tG)\tm_2(\tH),
\end{equation}
where $\cup_\bot$ denotes the union of fuzzy sets based on t-conorm $\bot$. The operator $\Cup_\bot$ is associative for any t-conorm. To be consistent with the conjunctive combination \eqref{eq:dempster_fuzzy}, it makes sense to choose the probabilistic sum $u \bot v = u + v -uv$, which is dual to the product t-norm.  \new{As already noted in Section \ref{subsec:poss}, disjunctive combination may be more appropriate than normalized conjunctive rules \eqref{eq:dempster_fuzzy1}-\eqref{eq:dempster_fuzzy} when the  evidence is highly conflicting, in which case the normalizations in \eqref{eq:Gamma_norm}, \eqref{eq:dempster_fuzzy1} and \eqref{eq:dempster_fuzzy} may result in numerical instability.}

\subsection{Belief and plausibility of fuzzy events} 
\label{subsec:belplfuzzy}

As plausibility and belief functions are, respectively, mixtures of possibility and necessity measures, definitions for the degree of belief and the degree of plausibility of a fuzzy event follow naturally from corresponding definitions in the possibilistic framework. As explained in Section \ref{subsec:poss}, there are two definitions for the possibility and the necessity of fuzzy events, based on Sugeno and Choquet integrals. Consequently, a belief function $Bel$ induced by a fuzzy mass function $\tm$ can be extended to the lattice $(\calF(\Theta),\wedge,\vee)$ in, at least, two ways, which are described below.

\paragraph{Extension based on Sugeno integrals} Based on the definitions of  the possibility,  necessity and guaranteed possibility of fuzzy events  \eqref{eq:poss_fuzzy}, \eqref{eq:nec_fuzzy} and \eqref{eq:fGP}, the belief,  plausibility and commonality functions induced by a  fuzzy mass function $m$ can be extended to fuzzy events as:
\begin{subequations}
\label{eq:sugeno}
\begin{align}
\label{eq:Belf}
Bel_\tm^{(S)}(\tA)&:=\sum_{i=1}^f m_i N_{\tF_i}^{(S)}(\tA)=\sum_{i=1}^f m_i \min_{\theta\in\Theta} \; (\tA\vee \tF_i^c)(\theta)\\
\label{eq:Plf}
Pl_\tm^{(S)}(\tA)&:=\sum_{i=1}^f m_i \Pi_{\tF_i}^{(S)}(\tA)=\sum_{i=1}^f m_i \max_{\theta\in\Theta}\; (\tA\wedge \tF_i)(\theta) ,
\end{align}
and
\begin{equation}
\label{eq:Qf}
Q_\tm^{(S)}(\tA):=\sum_{i=1}^f m_i \Delta_{\tF_i}^{(S)}(\tA)=\sum_{i=1}^f m_i \min_{\theta \in \Theta}(\tA^c \vee \tF)(\theta).
\end{equation}
\end{subequations}
Definitions \eqref{eq:Belf} and \eqref{eq:Plf} were first introduced by Zadeh in \cite{zadeh79}. As the necessity measures $N_{\tF_i}^{(S)}$ are belief functions in the lattice $(\calF(\Theta),\wedge,\vee)$, so is $Bel_\tm^{(S)}$. 

\paragraph{Extension based on Choquet integrals} Using now definitions \eqref{eq:poss_nec_fuzzy_Choquet} and \eqref{eq:fGP1} based on Choquet integrals, we get the following extensions of function $Bel_\tm$, $Pl_\tm$ and $Q_\tm$ to fuzzy events:
\begin{subequations}
\label{eq:choquet}
\begin{align}
\label{eq:Belf1}
Bel_\tm^{(C)}(\tA)&:=\int_0^1 Bel_\tm(\cut\alpha\tA)d\alpha=\sum_{i=1}^f m_i N_{\tF_i}^{(C)}(\tA)=\sum_{i=1}^f m_i Bel_{m_{\tF_i}}(\tA)\\
\label{eq:Plf1}
Pl_\tm^{(C)}(\tA)&:=\int_0^1 Pl_\tm(\cut\alpha\tA)d\alpha=\sum_{i=1}^f m_i \Pi_{\tF_i}^{(C)}(\tA)=\sum_{i=1}^f m_i  Pl_{m_{\tF_i}}(\tA) ,
\end{align}
and
\begin{equation}
\label{eq:Qf1}
Q_\tm^{(C)}(\tA):=\int_0^1 Q_\tm(\cut\alpha\tA)d\alpha=\sum_{i=1}^f m_i \Delta_{\tF_i}^{(C)}(\tA)=\sum_{i=1}^f m_i Q_{m_{\tF_i}}(\tA),
\end{equation}
\end{subequations}
where, as before, $m_{\tF_i}$ is the consonant mass functions such that $pl_{m_{\tF_i}}=\pi_{\tF_i}$. Definitions \eqref{eq:Belf1} and \eqref{eq:Plf1} were first proposed by Yen \cite{yen90}. 


We can remark that definitions \eqref{eq:sugeno} and \eqref{eq:choquet} coincide when either all focal sets are crisp, or the event $\tA$ is crisp. From \eqref{eq:poss_nec_fuzzy_Choquet1}, definitions \eqref{eq:choquet} also extend \eqref{eq:alphacut}, i.e., we have
\begin{subequations}
\label{eq:alphacut1}
\begin{align}
Bel_\tm^{(C)}(\tA) &= \int_0^1 Bel_{\cut\alpha\tm}(\tA) d\alpha,\\
Pl_\tm^{(C)}(\tA) &= \int_0^1 Pl_{\cut\alpha\tm}(\tA) d\alpha,
\end{align}
and
\begin{equation}
Q_\tm^{(C)}(\tA)=  \int_0^1 Q_{\cut\alpha\tm}(\tA) d\alpha
\end{equation}
\end{subequations}
for all $\tA \in \calF(\Theta)$. Overall, both  \eqref{eq:sugeno} and \eqref{eq:choquet} seem to provide sensible definitions to quantify the degrees of belief and plausibility of fuzzy events based on uncertain and fuzzy evidence. Yen \cite{yen90} argued for \eqref{eq:choquet} by showing that the Choquet integrals are more sensitive to small changes in the fuzzy focal sets $\tF_i$ than the Sugeno integrals. This  argument could actually be reversed, as membership functions of fuzzy sets are usually difficult to determine precisely in practice, and the relative insensitivity of the degrees of belief and plausibility to small perturbations of the membership functions of the focal sets could rather be considered as an advantage. On the other hand, the fact that, for instance, $Bel_\tm^{(C)}(\tA)$ can be determined by computing $Bel_{\cut\beta\tm}(\cut\alpha\tA)$ for all $\alpha$ and $\beta$, and then by integrating over $\alpha$ and $\beta$, is a nice property of the Choquet integral. More research is needed to find decisive arguments in favor of any of these two approaches.


\section{Application to statistical inference}
\label{sec:stat}

In the following, we apply the general framework outlined in Section \ref{sec:fuzzyDS} to statistical inference. The consonant  likelihood-based belief function, originally proposed by Shafer \cite{shafer76} is first  recalled in Section \ref{subsec:likelihood}, together with its axiomatic justification provided in \cite{denoeux13b}. In Section \ref{subsec:axioms}, we show that, by adding one axiom and enlarging the solution space, we can justify the representation of sample information by  a fuzzy mass function with a single focal set, equal to the relative likelihood function. Binomial inference is used as an illustrative example in Section \ref{subsec:ex}.

\subsection{Likelihood-based belief function}
\label{subsec:likelihood}

Let us consider a statistical inference problem in which the observable data $X \in \calX$ is randomly generated according to a  probability mass or density function $f(x;\btheta)$, where  $\btheta$ is a parameter whose value is only known to belong to a finite set $\Theta$. After observing a realization $x$ of $X$, we wish to represent the information about $\btheta$ by a belief function $Bel(\cdot;x)$ on $\Theta$. Different approaches to this problem have been proposed by several authors, generalizing either Bayesian inference  \cite{dempster66,dempster68a,dempster08} or  frequentist concepts such as confidence regions and p-values \cite{balch12,denoeux18b,martin16c,martin19}.  

A particularly simple and appealing solution, proposed by Shafer \cite{shafer76},  is to consider the \emph{consonant} belief function whose contour function is equal to the relative likelihood:
\begin{subequations}
\label{eq:lik}
\begin{equation}
pl_0(\theta;x):=\frac{L(\theta;x)}{\max_{\theta'\in\Theta} L(\theta';x)},
\end{equation}
where $L(\cdot;x): \theta \rightarrow f(x;\theta)$ is the likelihood function. The corresponding \new{consonant} plausibility and belief functions are, thus, defined as:
\begin{equation}
\label{eq:pllik}
Pl_0(A;x)=\max_{\theta\in A} pl_0(\theta;x)
\end{equation}
and
\begin{equation}
\label{eq:bellik}
Bel_0(A;x)=1-Pl_0(A^c;x)=1-\max_{\theta\not\in A} pl_0(\theta;x),
\end{equation}
\end{subequations}
for all $A \subseteq \Theta$. This likelihood-based belief function has several interesting properties discussed in \cite{denoeux13b}. In particular, combining $Bel_0(\cdot;x)$ with a Bayesian prior probability mass function (PMF) on $\Theta$ yields the Bayesian posterior PMF. 

The consonant belief function \eqref{eq:lik}, introduced in \cite{shafer76} on intuitive grounds, can be justified from the Least Commitment Principle  (LCP) \cite{denoeux13b}. This principle states that, when several belief functions satisfy some requirements for a reasonable representation of a given belief state, the least informative should be chosen  \cite{smets93b}. Here, we consider the following requirement, which states that combining  $Bel_\Theta(\cdot;x)$ with a prior PMF $\pi(\theta)$ by Dempster's rule yields the Bayesian posterior.

%

\begin{Req}[Compatibility with Bayesian inference] 
\label{ax:Bayes}
Let $\pi(\theta)$ be a prior PMF. Then
\begin{equation}
\label{eq:bayes}
\pi \oplus Bel(\cdot;x) = P(\cdot \mid x),
\end{equation}
where $P(\cdot \mid x)$ is the PMF on $\Theta$ defined by 
\[
P(\theta \mid x) = \frac{L(\theta ; x) \pi(\theta)}{\sum_{\theta'\in \Theta}L(\theta' ; x) \pi(\theta')}
\]
for all $\theta\in\Theta$.
\end{Req}

To apply the LCP, we need a way to compare the ``information content'' of belief functions. Here, we consider the \emph{Q-ordering relation} \cite{dubois86a} defined as follows: given two belief functions $Bel_1$ and $Bel_2$, with commonality functions $Q_1$ and $Q_2$, $Bel_1$ is Q-less committed than $Bel_2$ iff for all $A\subseteq \Theta$, $Q_1(A)\ge Q_2(A)$. The following proposition was proved in \cite{denoeux13b}.

\begin{Prop}
\label{prop:minQ}
Belief function $Bel_0(\cdot;x)$ defined by \eqref{eq:lik} is the Q-least committed belief function verifying Requirement \ref{ax:Bayes}. 
\end{Prop}
\begin{proof}
From Proposition \ref{prop:combBayes}, Requirement \ref{ax:Bayes} implies that
\[
Q(\{\theta\})=pl(\theta) = c L(\theta ; x),
\]
for all $\theta\in\Theta$ and for some constant $c$. The largest admissible value of $c$ is 
\[
c_0= [\max_{\theta\in\Theta} L(\theta ; x)]^{-1}.
\]
Hence, the Q-least committed belief function verifying Requirement \ref{ax:Bayes} must be such that
\[
Q(\{\theta\})=\frac{L(\theta;x)}{\max_{\theta'\in\Theta} L(\theta';x)}=pl_0(\theta;x).
\]
Now, for any commonality function $Q$, we have
\[
Q(A) \le \min_{\theta\in A} Q(\{\theta\})
\]
for all $A\subseteq \Theta$. Consequently,  we must have $Q(A) \le \min_{\theta\in A} pl_0(\theta;x)$ for all $A\subseteq \Theta$. The commonality function $Q_0$ corresponding to $Bel_0$ verifies $Q_0(A) = \min_{\theta\in A} pl_0(\theta;x)$. Consequently, $Q_0(A)\ge Q(A)$ for all $A$ and all $Q$ verifying Requirement \ref{ax:Bayes}.
\end{proof}

Proposition \ref{prop:minQ} means that $Bel_0(\cdot;x)$ is, in some sense, the least informative belief function verifying Requirement \ref{ax:Bayes}. If we see the belief function $Bel_0(\cdot;x)$ as being induced by a consonant mass function $m_0(\cdot;x)$, there is, however, a problem with this representation: if $x$ and $y$ are the outcomes of independent random experiments with the same parameter space $\Theta$, then $m_0(\cdot;x,y)$ is not equal to the orthogonal sum of $m_0(\cdot;x)$ and $m_0(\cdot;y)$, even though the two pieces of evidence $x$ and $y$ are independent. Indeed, it seems reasonable to impose the following requirement.

\begin{Req}[Combination of independent outcomes]
\label{ax:indep}
Let $x$ and $y$ be the outcomes of independent random experiments with the same parameter space $\Theta$.  Let  $m(\cdot;x)$,  $m(\cdot;y)$ and $m(\cdot;x,y)$ denote the mass functions on $\Theta$ induced, respectively, by the observation of $x$, $y$ and $(x,y)$. Then  
\[
m(\cdot;x,y)=m(\cdot;x) \oplus m(\cdot;y).
\]
\end{Req}

The fact that mass function $m_0(\cdot;x)$ fails to meet Requirement  \ref{ax:indep}  led Shafer to eventually reject it as a rational representation of statistical evidence \cite{shafer82}. It can be shown \cite{walley87,halpern92} that the only mass function $m(\cdot;x)$ verifying Requirements \ref{ax:Bayes} and \ref{ax:indep}, as well as the strong likelihood principle (i.e., depending only on the relative likelihood) and  some additional regularity properties, is Bayesian  and has the following expression:
\[
m(\{\theta\};x)=\frac{L(\theta;x)}{\sum_{\theta'\in \Theta} L(\theta';x)}.
\]
Hence, Requirement \ref{ax:indep} seems to rule out belief function $Bel_0$, as well as any other non-additive representation of statistical evidence, an argument used by Halpern and Fagin to question the usefulness of belief functions \cite{halpern92}. However, this conclusion does not hold if we enlarge the solution space to include fuzzy mass functions, as will be shown in the next section.

\subsection{Representation of statistical information as a fuzzy mass function}
\label{subsec:axioms}

Belief function $Bel_0(\cdot;x)$ defined by \eqref{eq:lik} is induced by a single crisp consonant mass function $m_0(\cdot;x)$, but it is also induced by fuzzy mass functions. In particular, it is induced by the fuzzy mass function $\tm_0(\cdot;x)$ such that $\tm_0(\tL_x;x)=1$, where $\tL_x$ is the fuzzy subset of $\Theta$ with membership function
\[
\tL_x(\theta):=\frac{L(\theta;x)}{\max_{\theta'\in\Theta} L(\theta';x)}=pl_0(\theta;x)
\]
for all $\theta\in \Theta$. Fuzzy set $\tL_x$ can be seen as the \emph{fuzzy set of likely values of $\btheta$ after observing $x$}. It is clear that $\tm_0(\cdot;x)$ meets Requirement \ref{ax:indep}. Indeed, if $x$ and $y$ are independent observations,
\[
\tL_{x,y}=\tL_x \varodot \tL_y
\]
and
\[
\tm_0(\cdot;x,y)=\tm_0(\cdot;x)\oplus\tm_0(\cdot;y).
\]
As a consequence of Proposition \ref{prop:combBayesfuzzy}, $\tm_0(\cdot;x)$ also meets Requirement \ref{ax:Bayes}. Fuzzy mass function $\tm_0(\cdot;x)$ thus seems to be an adequate representation of statistical evidence: it verifies both requirements, and the associated belief function is the Q-least committed subject to Requirement \ref{ax:Bayes}.

We can remark that there are other fuzzy mass functions that induce $Bel_0(\cdot;x)$. They are characterized in the following proposition.

\begin{Prop}
A fuzzy mass function $\tm$ on $\Theta=\{\theta_1,\ldots,\theta_q\}$ with focal sets $\tF_1,\ldots,\tF_f$ and masses $m_1,\ldots,m_f$ verifies $Bel_\tm=Bel_0$ if and only if
\begin{equation}
\label{eq:cond1}
\sum_{i=1}^f m_i \tF_i=\tL_x
\end{equation}
and there exists a permutation $\sigma$ of $\{1,\ldots,q\}$ such that, for all $i\in\{1,\ldots,f\}$,
\begin{equation}
\label{eq:cond2}
\tF_i(\theta_{\sigma(1)}) \le \tF_i(\theta_{\sigma(2)})\le \ldots \le \tF_i(\theta_{\sigma(q)}).
\end{equation}
\end{Prop}
\begin{proof}
From \eqref{eq:contour_fuzzy}, condition \eqref{eq:cond1} is necessary and sufficient to ensure that $pl_\tm(\theta)=Q_\tm(\{\theta\})=pl_0(\theta)$ for all $\theta\in\Theta$. Now, for any $A\subseteq \Theta$,
\[
Q_\tm(A)=\sum_{i=1}^f m_i \min_{\theta\in A} \tF_i(\theta)
\]
If \eqref{eq:cond2} holds, then, for any $A\subseteq \Theta$, there exists $\theta_A$ in $\Theta$ such that, for any $i\in\{1,\ldots,f\}$, $\tF_i(\theta_A)=\min_{\theta\in A} \tF_i(\theta)$. We then have
\[
Q_\tm(A)=\sum_{i=1}^f m_i\tF_i(\theta_A)=\min_{\theta\in A} \sum_{i=1}^f m_i\tF_i(\theta)=\min_{\theta\in A} Q_\tm(\{\theta\})=Q_0(A).
\]
Conversely, assume that $Q_\tm(A)=Q_0(A)=\min_{\theta\in A} pl_0(\{\theta\})$ for all $A\subseteq\Theta$. Let $\sigma$ be a permutation of $\{1,\ldots,q\}$ such that
\[
pl_0(\theta_{\sigma(1)}) \le pl_0(\theta_{\sigma(2)})\le \ldots \le pl_0(\theta_{\sigma(q)}).
\]
For any $k\in\{1,\ldots,q-1\}$, we have
\[
Q(\{\theta_{\sigma(k)}\})=Q(\{\theta_{\sigma(k)},\theta_{\sigma(k+1)}\}),
\]
i.e.,
\[
\sum_{i=1}^f m_i\tF_i(\theta_{\sigma(k)})=\sum_{i=1}^f m_i \min\left(\tF_i(\theta_{\sigma(k)}),\tF_i(\theta_{\sigma(k+1)})\right),
\]
which implies that $\tF_i(\theta_{\sigma(k)})\le \tF_i(\theta_{\sigma(k+1)})$ for  all $k$ and all $i$.
\end{proof}

There are, thus, infinitely many fuzzy mass functions $\tm$ that induce $Bel_0$. Let $\tm(\cdot;x)$ and $\tm(\cdot;y)$ be fuzzy mass functions corresponding to $Bel_0(\cdot;x)$ and $Bel_0(\cdot;y)$, assumed to have, respectively, $f$ and $f'$ focal sets. Then, $\tm(\cdot;x)\oplus\tm(\cdot;y)$ will generally have $ff'$ focal sets and will be different from $\tm(\cdot;x,y)$, except if $f=f'=1$. Consequently, Requirement \ref{ax:indep}  justifies the choice of $\tm_0$ as a fuzzy mass function representation of $Bel_0$, which is the Q-least committed belief function satisfying Requirement \ref{ax:Bayes}.

\subsection{Application to binomial inference}
\label{subsec:ex}

Let us consider an urn with a known number $N$ of balls and an unknown proportion $\btheta$ of black balls. Thus, $\btheta\in\Theta=\{0,1/N,2/N,\ldots,(N-1)/N,1\}$. Assume that we have observed $x$ black balls in $n$ draws with replacement. The likelihood function is
\[
L(\theta;x)= \begin{pmatrix} n\\ x\end{pmatrix} \theta^x(1-\theta)^{n-x}.
\]
Let $\thetah_x$ be the maximum likelihood estimate:
\[
\thetah_x=\arg \max_{\theta\in\Theta} L(\theta;x).
\]
The fuzzy set $\tL_x$ of likely values of $\btheta$ after observing $x$ has membership function
\[
\tL_x(\theta)=\frac{L(\theta;x)}{L(\thetah_x;x)}=\left(\frac{\theta}{\thetah_x}\right)^x\left(\frac{1-\theta}{1-\thetah_x}\right)^{n-x}.
\]
The statistical evidence is represented by mass function $\tm_0(\cdot;x)$ such that $\tm_0(\tL_x;x)=1$. By construction, combining $\tm_0(\cdot;x)$ by Dempster's rule with a Bayesian prior yields the Bayesian posterior. For instance, combining $\tm_0(\cdot;x)$ with a uniform prior  yields the posterior PMF
\[
p(\theta \mid x) \propto \theta^x(1-\theta)^{n-x}.
\]

If we now perform a second random experiment in which we observe $y$ black balls out of $q$ draws, the combined evidence about $\theta$ is represented by the fuzzy mass function
\[
m_0(\cdot,x,y)=m_0(\cdot,x)\oplus m_0(\cdot,y)
\]
with the single fuzzy focal set $\tL_{x,y}=\tL_x\varodot \tL_y$.

\paragraph{Frequentist properties} The previous analysis  applies \emph{after} the random experiment has been performed and the number $x$ of black balls has been observed. \emph{Before} the experiment has been performed, the number $X$ of black balls is a random variable with binomial distribution, and we can consider the random set $(\Omega_X,2^{\Omega_X},P_{X|\theta_0},\tGamma)$, where $\Omega_X=\{0,1,\ldots,n\}$ is the sample space of $X$, $P_{X|\theta_0}$ is the probability distribution of $X$ for the true value $\theta_0$ of the parameter, and $\tGamma(x)=\tL_x$. The corresponding fuzzy mass function $\tm$ is given by
\[
\tm(\tL_x)=P_{X|\theta_0}(\{x\})=\begin{pmatrix} n\\ x\end{pmatrix}\theta_0^x (1-\theta_0)^{n-x} 
\]
for all $x\in\Omega_X$. The plausibility $pl_\tm(\theta_0)$ of $\theta_0$ is a random variable that take values $\tL_x(\theta_0)$, for $x\in\Omega$, with probabilities  $P_{\theta_0}(\{x\})$. If $N$ is large enough so that $\btheta$ can be treated as a continuous parameter, we have, by Wilks' theorem \cite{wilks38},
\[
-2 \log pl_\tm(\theta_0) \cvgd \chi^2_1
\]
and $n\rightarrow \infty$. Consequently, 
\[
P(-2 \log pl_\tm(\theta_0)\le\chi^2_{1;1-\alpha})= P(pl_\tm(\theta_0)\ge \exp(-0.5\chi^2_{1;1-\alpha}))\rightarrow 1-\alpha,
\]
i.e., the $c_\alpha$-cut of $\tL_X$, with $c_\alpha=\exp(-0.5\chi^2_{1;1-\alpha})$, is an asymptotic confidence region for $\btheta$ at level $1-\alpha$. This property is illustrated in Figure \ref{fig:pl_theta0} for $\theta_0=0.3$ and three cases: (1) $N=100$, $n=10$, (2) $N=n=100$ and $N=n=1000$. The true confidence levels of regions $^{c_\alpha}\tL_X$ for $\alpha=0.01$, $\alpha=0.05$ and $\alpha=0.1$ are shown in Table \ref{tab:covprob}.

\begin{table}
\begin{center}
\caption{Coverage probabilities of confidence regions $^{c_\alpha}\tL_X$ for $\alpha=0.01$, $\alpha=0.05$ and $\alpha=0.1$, with $\theta_0=0.3$. \label{tab:covprob}}
\begin{tabular}{cccc}
 \cline{2-4}
 & \multicolumn{3}{c}{Coverage probability}\\
 \cline{2-4}
$1-\alpha$ & $N=100, n=50$& $N=n=100$ & $N=n=1000$\\
\hline
0.99 &0.9927 &0.9922 &0.9894\\
0.95 &0.9570 &0.9525 &0.9511\\
0.90 &0.8825 &0.9025 &0.8955\\
\hline
\end{tabular}
\end{center}
\end{table}

\begin{figure}
\centering  
\subfloat[\label{fig:pl_theta0_100_50}]{\includegraphics[width=0.5\textwidth]{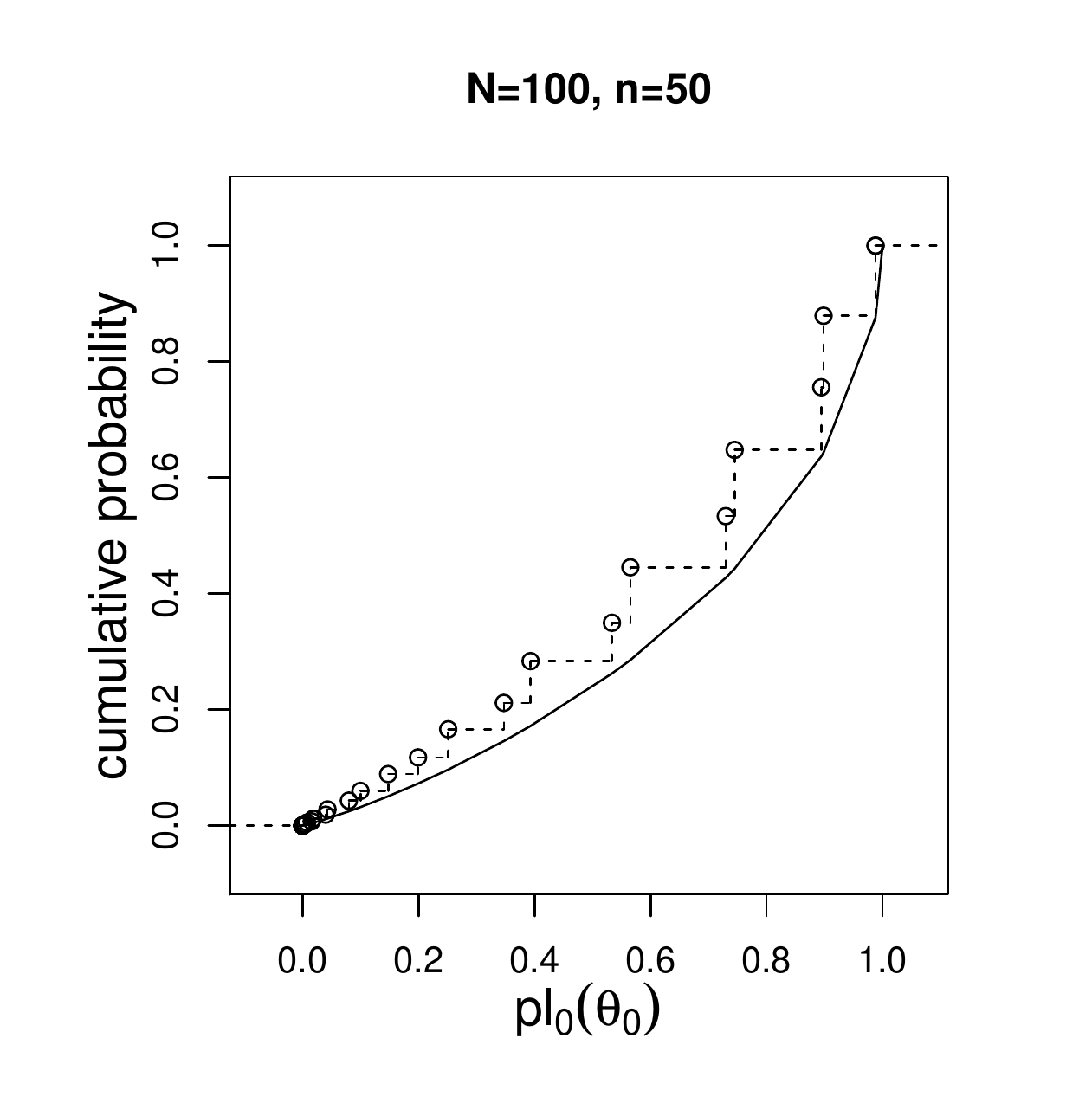}}
\subfloat[\label{fig:pl_theta0_100}]{\includegraphics[width=0.5\textwidth]{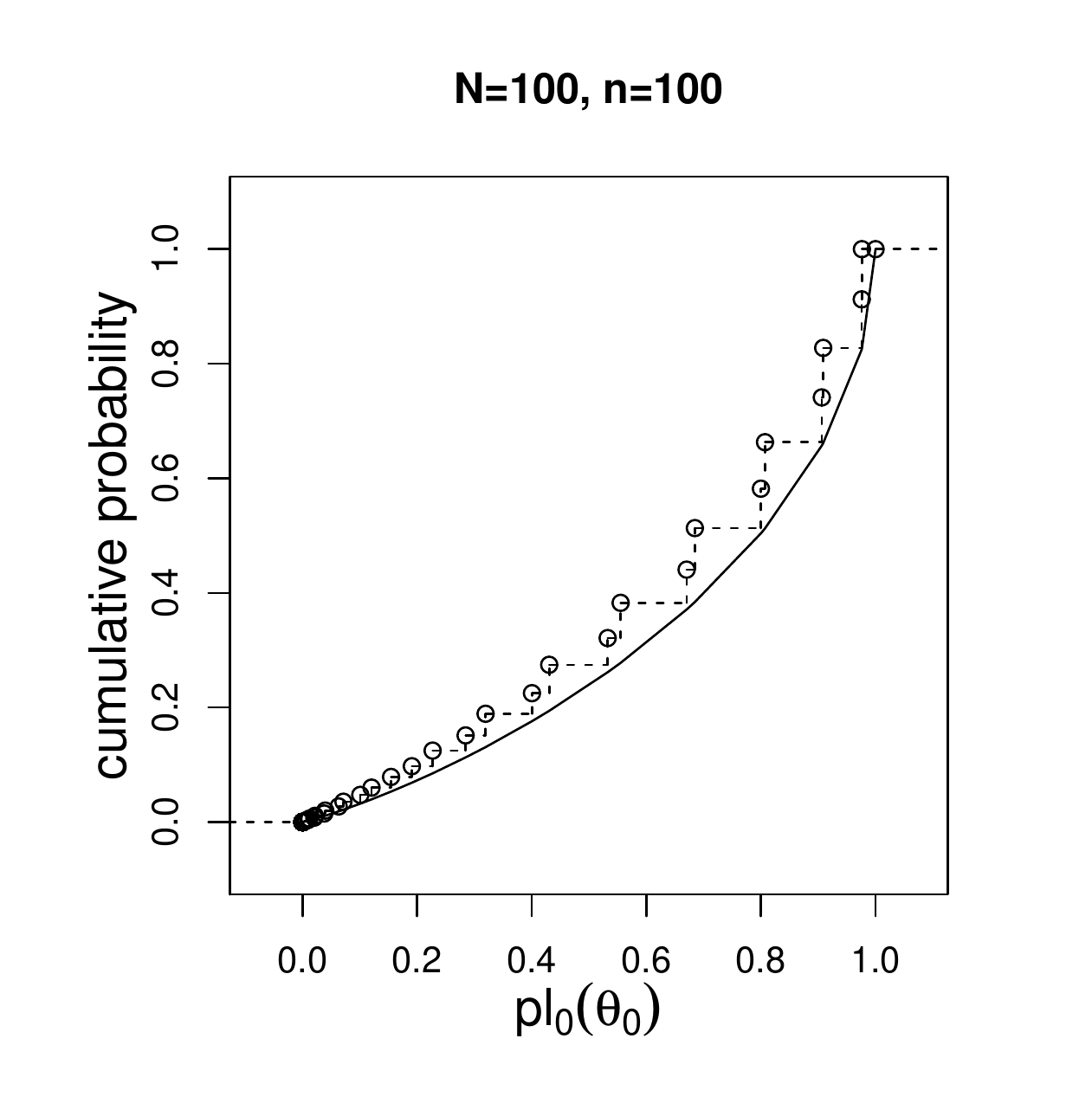}}\\
\subfloat[\label{fig:pl_theta0_1000}]{\includegraphics[width=0.5\textwidth]{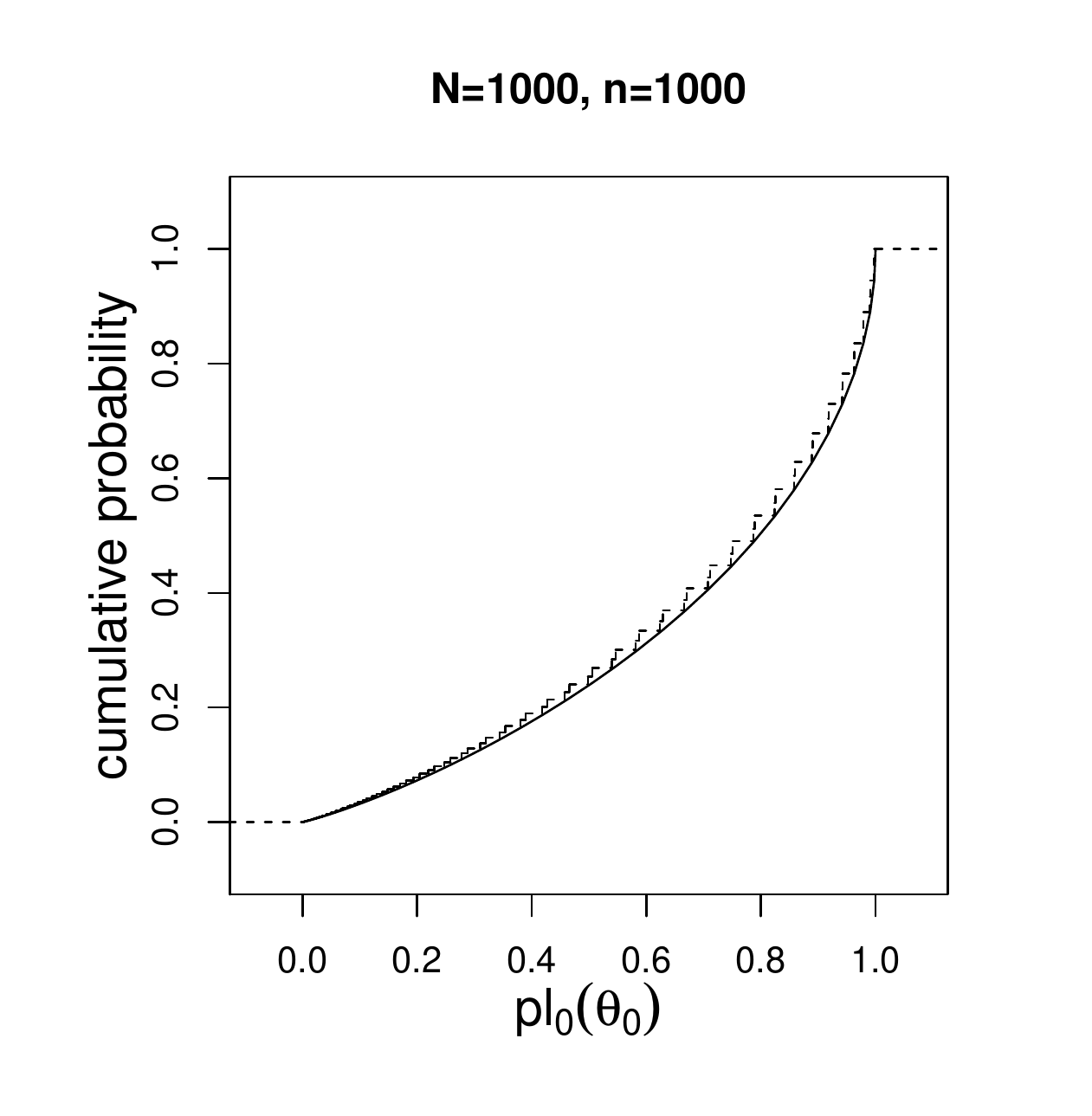}}
\caption{True cdf of $pl_0(\theta_0;X)=\tL_X(\theta)$ (broken lines) and asymptotic approximation (solid lines) for $N=100$ and $n=50$ (a), $N=n=100$ (b) and $N=n=1000$ (c). \label{fig:pl_theta0}}
\end{figure}

\paragraph{Prediction}

Assume that we have picked $x$ black balls out of $n$ draws with replacement, and we now want to predict the number $Y$ of black balls that will be obtained in another random experiment in which $r$ balls will be drawn with replacement from the same urn. This prediction problem was addressed in \cite{ ann14a,ann16} in the belief function framework. The approach introduced in \cite{ ann14a} is to write $Y$ as a function $\varphi(\btheta,\bU)$ of parameter $\btheta$ and a pivotal random variable $\bU$ with known probability distribution. A \emph{predictive belief function} on $Y$ is then obtained by combining the relation $Y=\varphi(\btheta,\bU)$ with the belief function on $\btheta$ derived from the likelihood function. Here, we can write
\begin{equation}
\label{eq:phi}
Y=\sum_{i=1}^r I(U_i \le \btheta),
\end{equation}
where $\bU=(U_1,\ldots,U_r)$ has a uniform distribution in $[0,1]^r$.

For a given realization $\bu=(u_1,\ldots,u_r)$ of $\bU$, the fuzzy constraint ``$\btheta \text{ is } \tL_x$'' induces a fuzzy constraint ``$Y \text{ is } \tY_{x,\bu}$'', where $\tY_{x,\bu}$ is a fuzzy subset of $\Omega_Y=\{0,\ldots,r\}$ whose membership function can be computed using the extension principle \cite{zadeh75} as
\[
\tY_{x,\bu}(y)=\max_{\{\theta: y=\sum_{i=1}^r I(u_i \le \theta)\}} \tL_x(\theta).
\]
The mapping $\bu \rightarrow \tY_{x,\bu}$ and the probability distribution of $\bU$ thus jointly define a random fuzzy set of $\Omega_Y$. The corresponding fuzzy mass function $\tmu_x$ is defined by
\[
\tmu_x(\tA)=P(\{\bu \in [0,1]^r: \tY_{x,\bu}=\tA\}),
\] 
for all $\tA \in \calF^*(\Omega_Y)$. 

It can be shown that the predictive belief function generated by $\tmu_x$ also has an interesting frequentist property \cite{denoeux18b}. As before, let $c_\alpha=\exp(-0.5\chi^2_{1;1-\alpha})$,  let ${^{c_\alpha}}\tmu_x$ be the (crisp) mass function whose focal sets are the $c_\alpha$-cuts of the fuzzy focal sets of $\tmu_x$, and let ${^{c_\alpha}}Bel_{\tmu_x}$ be the corresponding function. Then for any $\alpha\in (0,1)$, we have, asymptotically,
\begin{equation}
\label{eq:predictive_bf}
P_{X|\theta_0}\left( \left\{x\in \Omega_X : \forall A\subseteq \Omega_Y, \; {^{c_\alpha}}Bel_{\tmu_x}(A) \le P_{Y|\theta_0}(A)\right\}\right) \ge 1-\alpha,
\end{equation}
i.e., ${^{c_\alpha}}Bel_{\tmu_x}$ dominates the true probability distribution $P_{Y|\theta_0}$ of $Y$ for a proportion of the samples at least equal to $1-\alpha$. Function ${^{c_\alpha}}Bel_{\tmu_x}$ is, thus, a predictive belief function at confidence level $1-\alpha$, according to the terminology introduced in \cite{denoeux06b}.

As a numerical example, assume that we have observed $x=28$ black balls out of $n=100$ draws with replacement from an urn with $N=100$ balls, and we wish to predict the number $Y$ of black balls in $r=4$ future draws from the same urn. The belief and plausibility of any event $A\subseteq \Omega_Y$ can be approximated by Monte Carlo simulation. More precisely, we generate $K$ realizations  $\bu_1,\ldots,\bu_K$ of $\bU$, and we approximate  $Pl_{\tmu_x}(A)$ and $Bel_{\tmu_x}(A)$ and, for any $A\subseteq \Omega_Y$ as, respectively
\[
Pl_{\tmu_x}(A) \approx \frac1K \sum_{k=1}^K \max_{\theta\in A} \tY_{x,\bu_k}(\theta)
\]
and
\[
Bel_{\tmu_x}(A) \approx 1-\frac1K \sum_{k=1}^K \max_{\theta\not\in A} \tY_{x,\bu_k}(\theta).
\]
In particular, the contour function is approximated as
\[
pl_{\tmu_x}(\theta) \approx \frac1K \sum_{k=1}^K  \tY_{x,\bu_k}(\theta)
\]
Figure \ref{fig:contour_Y} shows a plot of the contour function $pl_{\tmu_x}$ together with the plug-in PMF $p_{Y|\thetah_x}$, with $\thetah_x=x/n=0.28$. Figure \ref{fig:cdfs_Y} displays the lower and upper cumulative distribution functions (CDFs) defined, respectively, as $F_*(y)=Bel_{\tmu_x}((-\infty,y])$ and $F^*(y)=Pl_{\tmu_x}((-\infty,y])$, together with the plug-in CDF $F_{Y|\thetah_x}(y)=P_{Y|\thetah}((-\infty,y])$. We can remark that, as $\tL_x(\thetah_x)=1$, we have
\[
Bel_{\tmu_x}(A) \le P_{Y|\thetah}(A) \le Pl_{\tmu_x}(A) 
\]
for all $A\subseteq\Omega_Y$, which is verified in Figure \ref{fig:cdfs_Y} for $A=(-\infty,y]$ with $y\in \Omega_Y$. 

\begin{figure}
\centering  
\subfloat[\label{fig:contour_Y}]{\includegraphics[width=0.5\textwidth]{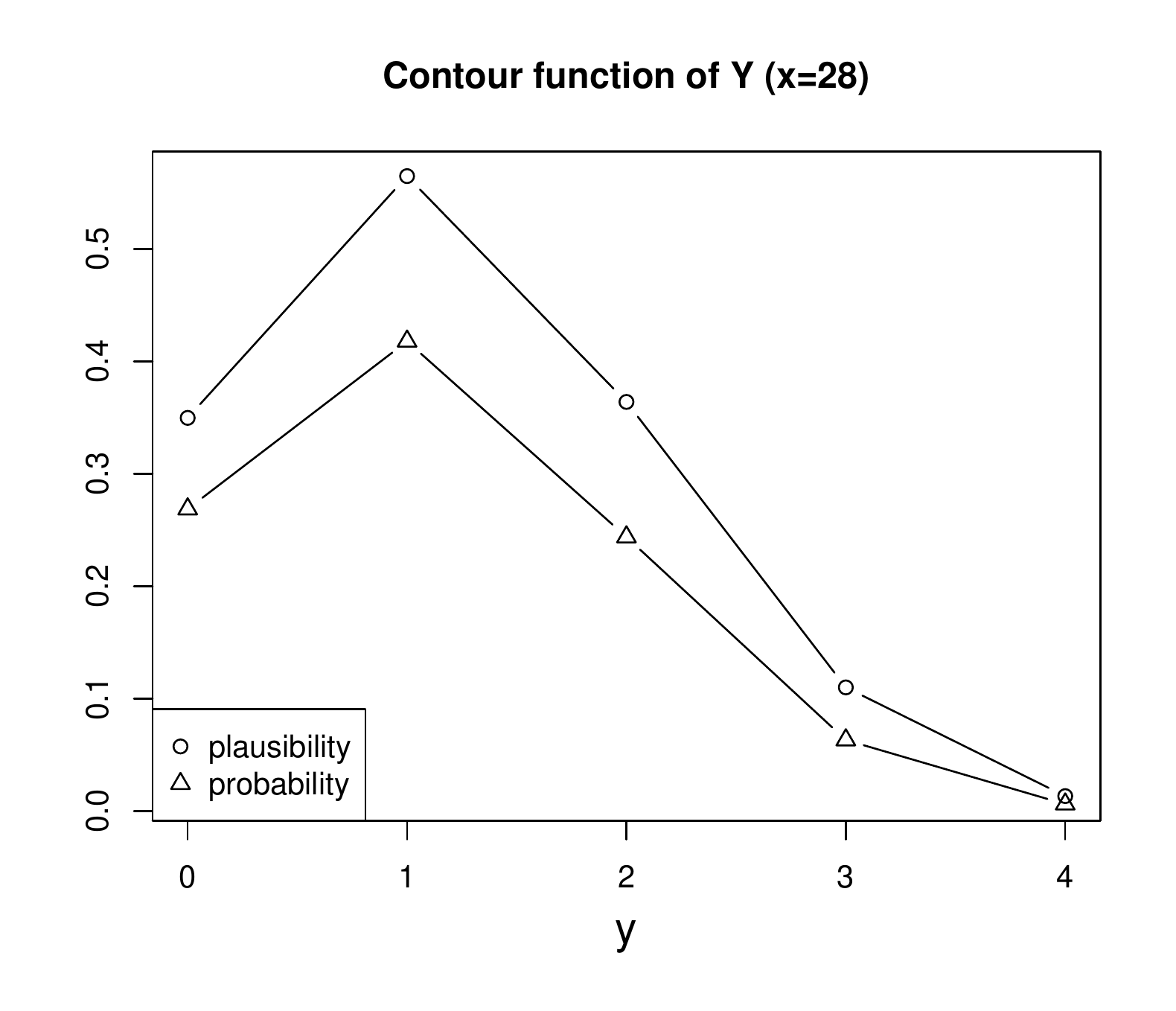}}
\subfloat[\label{fig:cdfs_Y}]{\includegraphics[width=0.5\textwidth]{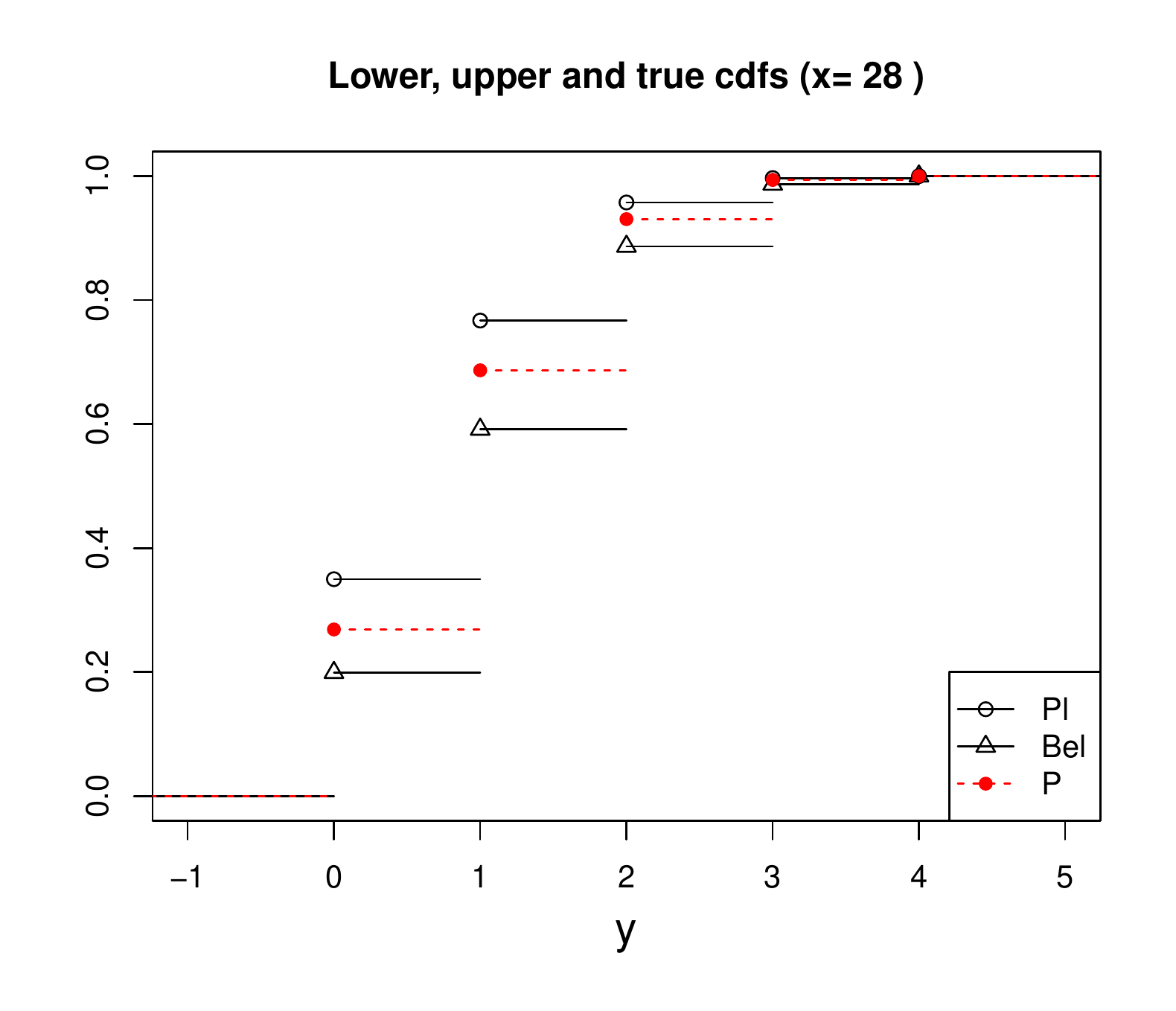}}
\caption{Views of the predictive belief function for $n=100$, $x=28$, $N=100$ and $r=4$. (a): Predictive contour function $pl_{\tmu_x}$ and true probability mass function $p_{Y|\theta_0}$ (with $\theta_0=0.3$); (b): Lower, upper and true CDFs. }
\end{figure}

To check the frequentist property \eqref{eq:predictive_bf} experimentally, we computed the coverage probability on the left-hand side of \eqref{eq:predictive_bf} for  $\alpha\in \{0.01,0.05,0.1\}$. In this expression, ${^{c_\alpha}}Bel_{\tmu_x}(A)$ can be approximated as
\[
{^{c_\alpha}}Bel_{\tmu_x}(A) \approx \frac1K \sum_{k=1}^K I( {^{c_\alpha}}\tY_{x,\bu_k} \subseteq A).
\]
We set $K=15,000$ in this simulation. The estimated  probabilities are shown in Table \ref{tab:cov_pred}. They are  close to the nominal values $1-\alpha$.

\begin{table}
\begin{center}
\caption{Estimated coverage probabilities of predictive belief functions  ${^{c_\alpha}}Bel_{\tmu_x}$  $\alpha=0.01$, $\alpha=0.05$ and $\alpha=0.1$, with $\theta_0=0.3$. \label{tab:cov_pred}}
\begin{tabular}{cccc}
\hline
$1-\alpha$ & 0.99 & 0.95 & 0.90 \\
\hline
Coverage probability & 0.9907 & 0.9496 & 0.8913\\
\hline
\end{tabular}
\end{center}
\end{table}

Figure \ref{fig:Pl_P_x4} shows the belief degrees ${^{c_\alpha}}Bel_{\tmu_x}(A)$ against the true probabilities $P_{Y|\theta_0}(A)$ for the $2^5=32$ subsets of $\Omega_Y$ and 9 values of $x$, with $\alpha=0.05$. We can verify that we have, for the most probable values of $x$ ($x\in \{22, 26, 30, 34, 38\}$), ${^{c_\alpha}}Bel_{\tmu_x}(A) \le P_{Y|\theta_0}(A)$ for all $A$.

\begin{figure}
\centering  
\includegraphics[width=\textwidth]{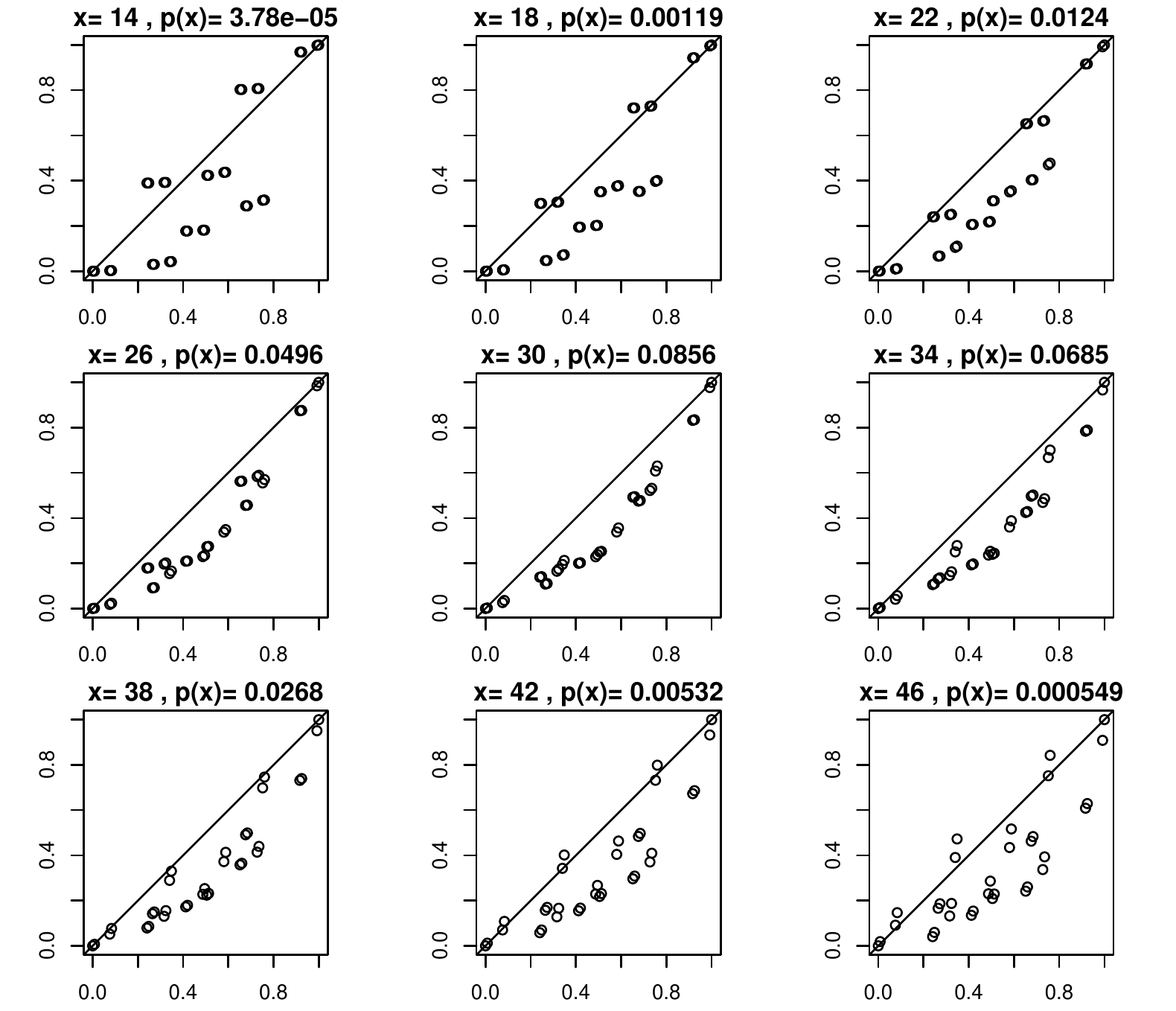}
\caption{Degrees of belief ${^{c_\alpha}}Bel_{\tmu_x}(A)$ (vertical axes) vs.  true probabilities $P_{Y|\theta_0}(A)$ (horizontal axes) or the $2^5=32$ subsets of $\Omega_Y=\{0,\ldots,4\}$ and 9 values of $x$, with $\alpha=0.05$.  \label{fig:Pl_P_x4}}
\end{figure}

\paragraph{Probability of fuzzy events} The formalism exposed in Section \ref{subsec:belplfuzzy} allows us to compute the belief and plausibility degrees of fuzzy events. For example, consider the fuzzy event: ``Most balls will be black'', where ``most'' corresponds to the following fuzzy subset of $\Omega_Y$:
\[
\tA:= \left\{\frac{0}{0}, \frac{1}{0}, \frac{2}{0.5}, \frac{3}{0.75}, \frac{4}{1}\right\}.
\]
Using  Zadeh's definitions \eqref{eq:sugeno} for the belief and plausibility of  fuzzy event based on the Sugeno integral, we get
\[
Bel^{(S)}_{\tmu_x}(\tA) \approx \frac1K \sum_{k=1}^K \min_{y\in \Omega_Y} (\tA \vee \tY_{x,\bu_k}^c)(y)
\]
and
\[
Pl^{(S)}_{\tmu_x}(\tA) \approx \frac1K \sum_{k=1}^K \max_{y\in \Omega_Y} (\tA \wedge \tY_{x,\bu_k})(y).
\]
For $x=28$, we find $Bel^{(S)}_{\tmu_x}(\tA) \approx 0.143$ and $Pl^{(S)}_{\tmu_x}(\tA) \approx 0.253$. Using definitions \eqref{eq:choquet} based on the Choquet integral, we have
\[
Bel^{(C)}_{\tmu_x}(\tA) = 0.5 Bel^{(C)}_{\tmu_x}(\{2,3,4\})+0.25 Bel^{(C)}_{\tmu_x}(\{3,4\})+0.25 Bel^{(C)}_{\tmu_x}(\{4\}) \approx 0.128
\]
and
\[
Pl^{(C)}_{\tmu_x}(\tA) = 0.5 Pl^{(C)}_{\tmu_x}(\{2,3,4\})+0.25 Pl^{(C)}_{\tmu_x}(\{3,4\})+0.25 Pl^{(C)}_{\tmu_x}(\{4\}) \approx 0.236.
\]

\section{Conclusion}
\label{sec:concl}

We have revisited Zadeh's notion of ``evidence of the second kind'' and outlined it importance as a generalization of both DS and possibility theories. According to this general view, possibility theory deals with belief functions generated by fuzzy sets, while DS theory deals with belief functions generated by random sets (or, equivalently, mass functions in the finite setting). Both theories fall under the umbrella of a more general formalism dealing with belief functions induced by random fuzzy sets, or fuzzy mass functions in the discrete case. This general formalism can be called ``Epistemic random fuzzy set theory'' (Figure \ref{fig:theories}). The adjective ``epistemic'' reminds us that, in this approach, a random fuzzy set is not seen as a model of a random mechanism for generating fuzzy data, but as a representation of evidence that can be both uncertain and fuzzy. 

\begin{figure}
\centering  
\includegraphics[width=0.5\textwidth]{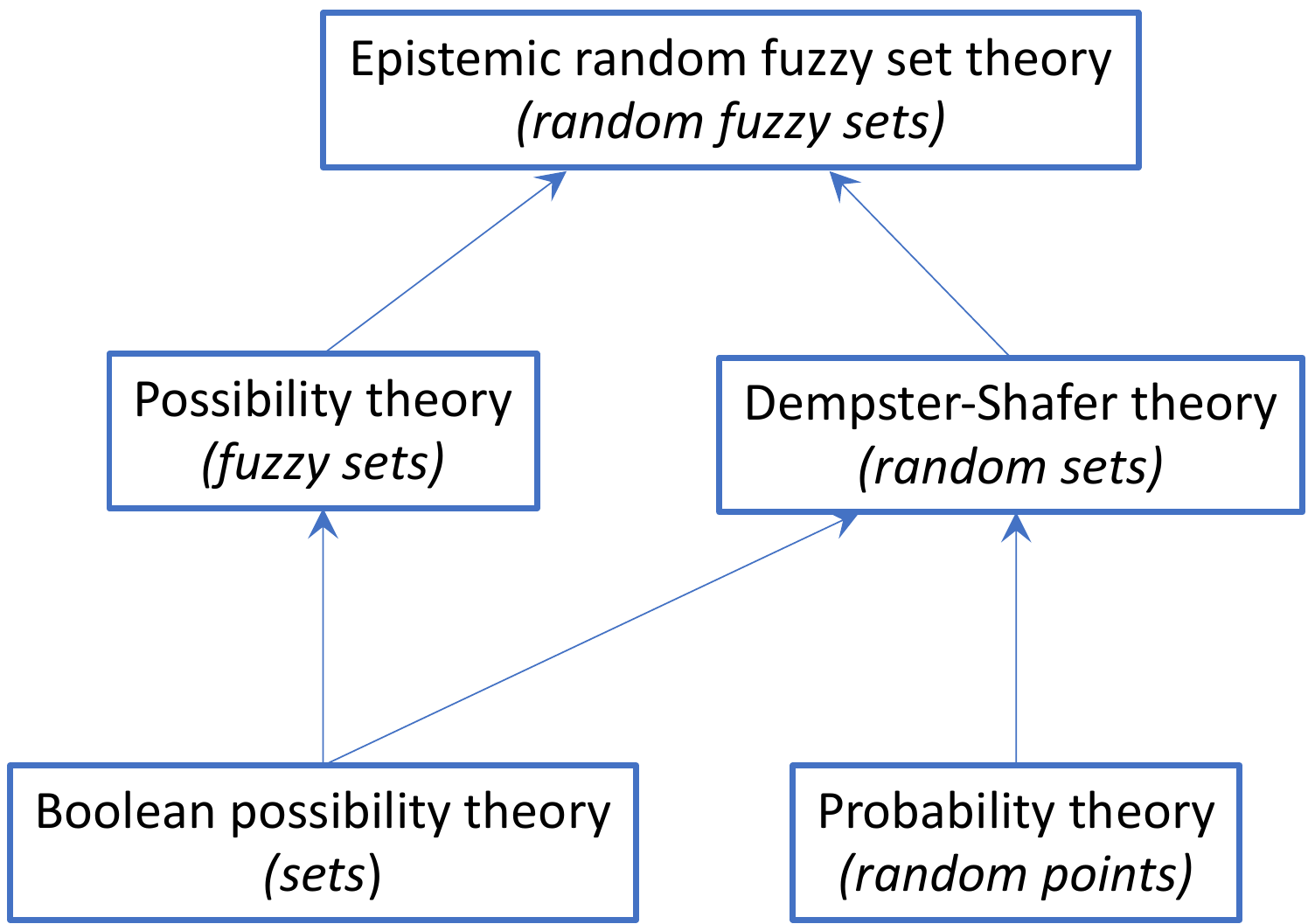}
\caption{Overview of different theories of uncertainty, from the less general ones (bottom) to the most general (top). \label{fig:theories}}
\end{figure}

In this general approach, a belief function is only a partial representation of evidence, which is  fully represented by the underlying random fuzzy set. In particular, two independent pieces of evidence must be combined at the random-fuzzy-set level. For instance, two possibility measures (or consonant belief functions) should be combined using the  normalized product rule of possibility theory if they are generated by fuzzy sets, and by  Dempster's rule is they are generated by random sets. These two combination mechanisms are special cases of a generalized product-intersection rule, which allows us to combine independent pieces of evidence that are both uncertain and fuzzy. 

As an application of this general framework, we have considered statistical inference. Shafer \cite{shafer76} first proposed to represent statistical evidence by the consonant belief function such that the plausibility of each single parameter value equals its relative likelihood. However, this representation was soon found to be problematic, as the belief function induced by the union of two independent samples is not equal to the orthogonal sum of the belief functions induced by each of the samples. This problem actually occurs only if we see the likelihood-based belief function as being generated by a crisp consonant mass function, i.e., if we see statistical evidence as  uncertain. The problem disappears if we see statistical evidence as fuzzy,  the relative likelihood function being interpreted as the membership function of the fuzzy set of likely values of the parameter. This perspective is consistent with the intuition of several authors who noticed the possibilistic nature of the concept of likelihood (see, e.g., \cite{smets82}). Possibility theory alone, however, is not expressive enough to support the development of a theory of statistical inference as it does not allow us to represent, for instance, the combination of the likelihood with a Bayesian prior. The more general epistemic random fuzzy set theory allows us to reconcile the possibilistic interpretation of likelihood with Bayesian inference, and  is rich enough to accommodate both ``likelihoodist'' and Bayesian views of statistical inference.

Throughout this paper, we have only considered finite sets to keep the level of mathematical exposition elementary. However, the whole framework developed in this paper can be extended to infinite spaces such as $\reels^n$ without much difficulty, using existing results on random fuzzy sets in a more general setting (see, e.g., \cite{couso11} and references therein). The practical application of such models to statistical inference or other uncertainty quantification problems will require the use of Monte Carlo simulation and numerical approximation techniques such as developed in \cite{ann16, sui18, denoeux19e}. Results in this direction will be reported in future publications.

\section*{References}

\end{document}